\documentclass{article} %
\usepackage{iclr2024_conference,times}

\usepackage{hyperref}
\usepackage{url}

\usepackage[utf8]{inputenc} %
\usepackage[T1]{fontenc}    %
\usepackage{hyperref}       %
\usepackage{url}            %
\usepackage{enumitem}
\usepackage{booktabs}       %
\usepackage{amsfonts}       %
\usepackage{amsmath,amssymb}
\usepackage{amsthm}
\usepackage{mathtools}
\usepackage{nicefrac}       %
\usepackage{microtype}      %
\usepackage{xcolor}         %
\usepackage{xspace}
\usepackage{xparse}         %
\usepackage{cleveref}       %
\usepackage{stmaryrd}       %
\usepackage{graphicx} 
\usepackage{wrapfig}
\usepackage{lscape}
\usepackage{rotating}

\usepackage{subcaption}    %
\usepackage{bm}
\usepackage{layouts}
\usepackage{soul}
\usepackage{multirow}

\usepackage{etoolbox,siunitx}

\usepackage{algorithm}
\usepackage{algpseudocode}

\robustify\bfseries
\graphicspath{{./figures/}}

\usepackage{thmtools,thm-restate}

\newtheorem{theorem}{Theorem}
\newtheorem{lemma}{Lemma}
\newtheorem{remark}{Remark}
\newtheorem{definition}{Definition}
\newtheorem{proposition}{Proposition}

\DeclareMathOperator*{\esssup}{\ensuremath{\text{\rm ess\,sup}}}

\DeclareMathOperator*{\argmax}{\ensuremath{\text{\rm arg\,max}}}
\DeclareMathOperator*{\range}{\ensuremath{\text{\rm Im}}}

\DeclareMathOperator{\tr}{\ensuremath{\text{\rm tr}}}

\DeclareMathOperator*{\rank}{\ensuremath{\text{\rm rank}}}
\DeclareMathOperator*{\Span}{\ensuremath{\text{\rm span}}}
\providecommand{\norm}[1]{\big\lVert#1\big\Vert}

\providecommand{\abs}[1]{\lvert#1\rvert}
\newcommand{\scalarp}[1]{{\left\langle #1\right\rangle}}

\newcommand{\R}{\mathbb R}
\newcommand{\C}{\mathbb C}
\newcommand{\N}{\mathbb N}

\newcommand{\EE}{\ensuremath{\mathbb E}}
\newcommand{\PP}{\ensuremath{\mathbb P}}
\newcommand{\Id}{I}

\newcommand{\sigalg}{\Sigma_{\spX}}
\newcommand{\rvX}{X}
\newcommand{\rvY}{X'}

\newcommand{\spX}{\mathcal{X}}
\newcommand{\spY}{\mathcal{X}}

\newcommand{\LiiX}{L^2_\mX(\spX)}
\newcommand{\LiiY}{L^2_{\mY}(\spY)}

\newcommand{\mX}{\mu} %
\newcommand{\mY}{\mu'} %
\newcommand{\mXt}[1]{\mu_{#1}}

\newcommand{\mXY}{\rho} %

\newcommand{\im}{\pi} %
\newcommand{\jim}{\rho} %

\newcommand{\Lii}{L^2_{\im}(\spX)}
\newcommand{\Wii}{W^{1,2}_\im(\spX)}

\newcommand{\TO}{\mathcal{T}}  %
\newcommand{\LO}{\mathcal{L}}  %
\newcommand{\ELO}{\widehat{\mathcal{L}}}  %
\newcommand{\ETO}{\widehat{\mathcal{T}}}  %

\newcommand{\param}{w}
\newcommand{\spParam}{\mathcal{W}}

\newcommand{\embXp}[1]{\psi_{\param#1}}

\newcommand{\embYp}[1]{\psi_{\param#1}'}

\newcommand{\embMXp}{\widehat{\Psi}_w}
\newcommand{\embMYp}{\widehat{\Psi}'_w}

\newcommand{\embMXpT}{\widehat{\Psi}^\top_w}
\newcommand{\embMYpT}{\widehat{\Psi}^{\prime\top}_w}

\newcommand{\spH}{\mathcal{H}}

\newcommand{\kHX}{k_{\rvX}}
\newcommand{\spHX}{\mathcal{H}}
\newcommand{\fHX}{\psi}

\newcommand{\kHY}{k_{\rvY}}
\newcommand{\spHY}{\mathcal{H}'}
\newcommand{\fHY}{\psi'}

\newcommand{\kHXp}{k_{\rvX}^{\param}}
\newcommand{\spHXp}{\mathcal{H}_{\param}}
\newcommand{\fHXp}{\phi_{\rvX}^{\param}}

\newcommand{\kHYp}{k_{\rvY}^{\param}}
\newcommand{\spHYp}{\mathcal{H}_{\param}'}
\newcommand{\fHYp}{\phi_{\rvY}^{\param}}

\newcommand{\Enc}[1]{E_{#1}}
\newcommand{\Dec}[1]{D_{#1}}

\newcommand{\bcon}{c_\psi}
\newcommand{\bconbis}{c_{\psi'}}

\newcommand{\proj}[1]{P_{#1}}

\newcommand{\Cx}{C_{\rvX}}
\newcommand{\Cy}{C_{\rvY}}
\newcommand{\Cxy}{C_{\rvX\!\rvY}}

\newcommand{\Cxp}{{C}^{\param}_{\rvX}}
\newcommand{\Cyp}{{C}^{\param}_{\rvY}}
\newcommand{\Cxyp}{{C}^{\,\param}_{\rvX\!\rvY}}

\newcommand{\Cxyd}{{C}^{\,\param}_{X\!\partial}}

\newcommand{\EEstim}{\widehat{T}}  %

\newcommand{\Data}{\mathcal{D}}

\newcommand{\ECxp}{\widehat{C}^{\param}_{\rvX}}
\newcommand{\ECyp}{\widehat{C}^{\param}_{\rvY}}
\newcommand{\ECxyp}{\widehat{C}^{\,\param}_{\rvX\!\rvY}}
\newcommand{\ECxyd}{\widehat{C}^{\,\param}_{X\!\partial}}

\newcommand{\TS}{S_{\im}}

\newcommand{\TSX}{S_{\mX}}  %
\newcommand{\TSY}{S_{\mY}}  %

\newcommand{\ESX}{\widehat{S}}  %
\newcommand{\ESY}{\widehat{S}'}  %

\newcommand{\spF}{\mathcal{F}} %

\newcommand{\Reg}{\mathcal{R}} %

\newcommand{\hnorm}[1]{\norm{#1}_{\rm{HS}}}

\newcommand{\reg}{\gamma}
\newcommand{\rate}{\varepsilon}

\newcommand{\score}{\mathcal{P}}
\newcommand{\rscore}{\mathcal{S}}
\newcommand{\escore}{\widehat{\mathcal{P}}}
\newcommand{\erscore}{\widehat{\mathcal{S}}}

\newcommand{\erefun}{\widehat{f}}

\newcommand{\elefun}{\widehat{g}}

\newcommand{\sval}{\sigma}
\newcommand{\eval}{\lambda}
\newcommand{\eeval}{\widehat{\lambda}}

\newcommand{\erevec}{\widehat{v}}
\newcommand{\elevec}{\widehat{u}}

\algrenewcommand\algorithmicrequire{\textbf{Input:}}
\algrenewcommand\algorithmicensure{\textbf{Output:}}

\usepackage[disable,textsize=tiny]{todonotes}

\newcommand{\riccardo}[2][]{\todo[color=violet!20,#1]{{\bf RG:} #2}}

\title{Learning invariant representations of time- \\homogeneous stochastic dynamical systems}

\author{Vladimir R. Kostic\thanks{Equal contribution, corresponding authors.}\\
Istituto Italiano di Tecnologia\\
University of Novi Sad\\
\texttt{vladimir.kostic@iit.it}
\And
Pietro Novelli\footnotemark[1]\\
Istituto Italiano di Tecnologia\\
\texttt{pietro.novelli@iit.it}
\And
Riccardo Grazzi\\
Istituto Italiano di Tecnologia\\
University College London\\
\And
Karim Lounici\\
CMAP \'Ecole Polytechnique\\
\And
Massimiliano Pontil \\
Istituto Italiano di Tecnologia\\
University College London\\
}

\iclrfinalcopy %
\begin{document}

\maketitle

\vspace{-.15truecm}
\begin{abstract}
We consider the general class of time-homogeneous stochastic dynamical systems, both discrete and continuous, and study the problem of learning 
a representation of the state that faithfully captures its dynamics. This is instrumental to learning the transfer operator or the generator of the system, which in turn can be used for numerous tasks, such as forecasting and interpreting the system dynamics.
We show that the search for a good representation can be cast as an optimization problem over neural networks. Our approach is supported by recent results in statistical learning theory, highlighting the role of approximation error and metric distortion in the learning problem.
The objective function {we propose} is associated with projection operators {from the} representation space {to the} data space, overcomes metric distortion, and can be empirically estimated from data. 
In the discrete-time setting, we further derive a relaxed objective function that is differentiable and numerically well-conditioned. We compare our method against state-of-the-art approaches on different datasets, showing better performance across the board.
\end{abstract}

\section{Introduction}\label{sec:intro}
\vspace{-.1truecm}Dynamical systems are a mathematical framework describing the evolution of state variables over time. These models, often represented by nonlinear differential equations (ordinary or partial) and possibly stochastic, have broad  applications in science and engineering, ranging from
climate sciences \citep{cannarsa2020mathematical,fisher2009data}, to finance~\citep{Pascucci2011}, to atomistic simulations \citep{mardt2018vampnets, McCarty2017,Schutte2001}, and to open quantum system dynamics~\citep{Gorini1976,Lindblad1976}, among others.
\vspace{-.05truecm}
However, it is usually the case that no analytical models of the dynamics are available and one must resort to data-driven techniques to characterize a dynamical system.  
Two powerful paradigms have emerged: deep neural networks (DNN) and kernel methods.  
The latter are backed up by solid statistical guarantees~\citep{Steinwart2008}  %
{determining when} linearly parameterized models can be learned efficiently. 
{Yet,}
selecting an appropriate kernel function 
may be a hard task, requiring a significant amount of experimentation and expertise. 
In comparison, the former are very effective in learning complex data representations \citep{goodfellow2016deep}, and benefit from a solid ecosystem of software tools, making the learning process feasible on large-scale systems. However, their statistical analysis is still in its infancy, with only a few {proven} results on their generalization properties. 

\vspace{-.1truecm}
Kernel-based methods hinge on the powerful idea of characterizing dynamical systems by lifting their definition over a Hilbert space of functions %
and then studying the associated \emph{transfer operators}. They describe the average evolution of functions of the state  ({\it observables}) over time and for deterministic systems are also known as \emph{Koopman operators}.  %
Furthermore, transfer operators are {\em linear}, and under additional assumptions admit a spectral decomposition, 
which provides a valuable tool to interpret and analyze the behavior of non-linear systems %
\citep[see e.g.][and references therein]{Brunton2022,Kutz2016}. The usefulness of this approach critically relies on an appropriate choice of the \textit{observable space}. In particular, in order to fully benefit from the spectral decomposition, {it is crucial} to find an observable space $\spF$ which linearizes the dynamics and is {\it invariant} under the action of the transfer operator. In kernel-based methods, such a space is implicitly linked to the kernel function and gives a clear mathematical meaning to the problem of choosing a ``good'' kernel. Unfortunately, the analytical form of transfer operators is often intractable or unavailable, especially in complex or poorly understood systems, posing challenges in constructing invariant representations.
\begin{figure}[t!]
  \centering \includegraphics[width=0.85\textwidth]{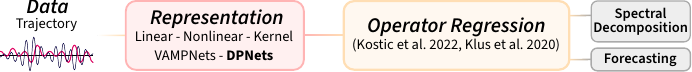}
    \caption{Pipeline for learning dynamical systems. 
    DPNets learn a data representation to be used with standard operator regression methods. 
    In turn, these are employed to solve downstream tasks such as forecasting and interpreting dynamical systems via spectral decomposition.}
    \label{fig:pipeline}
\vspace{-.67truecm}
\end{figure}

In this paper, we build synergistically upon both kernel and DNN paradigms:  
we first employ DNNs
to learn
an invariant representation that fully captures the system dynamics, and then forward this representation to kernel-based algorithms for the actual transfer operator regression task. 
This general framework is illustrated in Fig.~\ref{fig:pipeline}. Our method, named Deep Projection Networks (DPNets), addressing the challenge of providing good representations to the operator regression algorithms, can be cast as an optimization problem over neural networks and can benefit from a differentiable and numerically well-conditioned score functional, enhancing the stability of the training process.

\vspace{-.1truecm}
{\bf Previous work~} Extensive research has been conducted on learning dynamical systems from data. The monographs \citep{Brunton2022,Kutz2016} are standard references in this field. To learn transfer operators we mention the works~\citep{Alexander2020,Bouvrie2017,Das2020,Kawahara2016,Klus2019,Kostic2022,Williams2015b} presenting kernel-based algorithms, and~\citep{azencot2020forecasting, Bevanda2021,Fan2021,Lusch2018, Morton2018} based on deep learning schemes. 
Finding meaningful representations of the state of the system to be used in conjunction with transfer operator learning is a critical challenge, tackled by many authors. We mention the 
works~\citep{azencot2020forecasting, federici2024latent,Lusch2018,Morton2018,Otto2019} where a representation is learned via DNN schemes, as well as~\citep{Kawahara2016, Li2017, Mardt2019, mardt2018vampnets, Tian2021,Yeung2019,Wu2019} addressing the problem of learning invariant subspaces of the transfer 
operator.  
Mostly related to our methodology is 
\citep{andrew2013deep}, which introduced deep canonical correlation analysis and, especially, VAMPnets \citep{mardt2018vampnets, Wu2019}, which repurposed 
this approach to learn transfer operators. {Within our theoretical framework, indeed, we will show that VAMPNets can be recovered as a special case in the setting of discrete dynamics. More in general, operator learning with DNN approaches is reviewed e.g. in \cite{Kovachki2023}, with a specific focus on PDEs.} 

\vspace{-.1truecm}
{\bf Contributions~}   
Our contributions are:
{\bf 1)} Leveraging recent results in statistical learning theory, we formalize the problem of representation learning for dynamical systems and design a score function based on the orthogonal projections of the transfer operator in data spaces (Sec.~\ref{sec:prob}); 
{\bf 2)} 
We show how to reliably solve the corresponding optimization problem.
In the discrete case, our score
leads to optimal representation learning under the mild assumption of compactness of the transfer operator (Thm.~\ref{thm:main}). In the continuous case, our method applies to time reversal invariant dynamics (Thm.~\ref{thm:main_cont}) including the important case of Langevin dynamics;
{\bf 3)} We show how the learned representation can be used within the framework of operator regression in a variety of settings. Numerical experiments illustrate the versatility and competitive performance of our approach against several baselines. 
\vspace{-.1truecm}
{\bf Notation~} We let $\N$ be the set of natural numbers and $\N_0 \,{=}\, \{0\} \cup \N$. 
For $m\,{\in}\,\N$ we denote $[m]:=\{1,\ldots,m\}$. If $\TO$ is a compact linear operator on two Hilbert spaces we let $\sigma_i(\TO)$ be its $i$-th largest singular value, we let $(\TO)^{\dagger}$ be the Moore-Penrose pseudo-inverse, and $\TO^*$ the adjoint operator. We denote by $\|\cdot\|$ and  $\|\cdot\|_{\rm HS}$ the operator and Hilbert-Schmidt norm, respectively. 

\section{Representation learning for dynamical systems}\label{sec:prob}
\vspace{-.1truecm}
In this section, we give some background on transfer operators for dynamical systems, discuss key challenges in learning them from data, and propose our representation learning approach. 
\vspace{-.1truecm}

{\bf Background~} Let $(X_{t})_{t\in\mathbb{T}}$ be a stochastic process taking values in some state space $\spX$, where the time index $t$ can be either discrete ($\mathbb{T}=\N_0$) or continuous ($\mathbb{T}=[0,+\infty)$), and denote by $\mXt{t}$ the law of $X_t$.   
Let $\spF \subset \R^{\spX}$ be a prescribed space of real-valued functions, henceforth referred to as observables space. By letting $s,t \in \mathbb{T}$, with $s \leq t$, the \emph{forward transfer operator}  {$\TO_{s,t}\colon\spF\to\spF$ evolves  an observable $f: \spX \rightarrow \mathbb{R}$} from time $s \in \mathbb{T}$ to time $t$, by the conditional expectation
\begin{equation}\label{eq:transfer-operators}
[\TO_{s,t}( f )](x) 
:= \EE[ f(X_{t})\,\vert X_{s} = x],\;x\in\spX. 
\end{equation}
For a large class of stochastic dynamical systems, these \emph{linear} operators are {\em time-homogeneous}, that is they only depend on the time difference $t - s$. In this case $\TO_{s,t} = \TO_{0,t - s} =: \TO_{t - s}$. Further, we will use the shorthand notation $\TO := \TO_1$. 
Time-homogeneous transfer operators at different times are related through the Chapman-Kolmogorov equation~\citep{allen2007}, 
$\TO_{s + t} = \TO_{t}\,\TO_{s}$, implying that the family $(\TO_{t})_{t\in\mathbb{T}}$ forms a semigroup.
The utility of transfer operators is that, when the space $\spF$ is suitably chosen, they {\em linearize} the process. Two key requirements are that $\spF$ is {\it invariant} under the action of $\TO_{t}$, that is $\TO_{t} [\spF] \subseteq \spF$ for all $t$ (a property that we tacitly assumed above), and \textit{rich enough} to represent the flow of the process. %
Two common choices for $\spF$ that fulfill both requirements are the space of bounded Borel-measurable functions and the space of square-integrable functions $\Lii$ with respect to the invariant distribution $\im$, when one exists. To ease our presentation we first focus on the latter case, and then extend our results to non-stationary time-homogeneous processes. Formally, the process $(X_t)_{t\in\mathbb{T}}$ admits an \emph{invariant distribution} $\pi$ when $\mXt{s}=\im$ implies $\mXt{t}=\im$ 
for all $s\leq t$ \citep{ross1995stochastic}.~This in turn allows one to define the transfer operator on $\spF = \Lii$.

While in the discrete-time setting the whole process can be studied only through $\TO=\TO_{1}$, when time is continuous the process is characterized by the infinitesimal generator of the semigroup $(\TO_{t})_{t\geq 0}$, defined as 
\vspace{-.2truecm}
\begin{equation}\label{eq:generator}
\LO\ := \lim_{t\to 0^+} (\TO_{t} - \Id)/t.
\end{equation}
$\LO$ can be properly defined on the Sobolev space $\Wii \subset \Lii$ formed by functions with square-integrable gradients~\citep[see][]{Lasota1994,ross1995stochastic}.

{\bf Learning transfer operators~}
In practice, dynamical systems are only observed, and neither $\TO$ nor its domain $\spF \,{=}\, \Lii$ are known, providing a key challenge to learn them from data. 
The most popular algorithms~\citep{Brunton2022,Kutz2016} aim to learn the action of $\TO\colon \spF\,{\to}\,\spF$ on a predefined Reproducing Kernel Hilbert Space (RKHS) $\spH$, forming a subset of functions in $\spF$. 
This allows one, via the kernel trick, to formulate the problem of learning the {\em restriction} of $\TO$ to $\spHX$,
$\TO_{\vert_{\spHX}}\colon\spH\to\spF$, via empirical risk minimization
~\citep{Kostic2022}.  
However, recent theoretical advances~\citep{Korda2017,Klus2019,nuske2023finite}, proved that such algorithms are statistically consistent only to $\proj{\spHX} \TO_{\vert_{\spHX}}$, where $\proj{\spHX}$ is the orthogonal projection onto the closure of $\spHX$ in $\spF$. The projection $\proj{\spHX}$ constrains the {\em evolved} observables back inside $\spHX$, thereby, in general, altering the dynamics of the system. Therefore, to assure that one properly learns the dynamics, two major requirements on $\spH$ are needed: i) $\TO_{\vert_{\spHX}}$ needs to approximate well $\TO$, i.e. the space $\spH$ needs to be big enough relative to the domain of $\TO$; ii) the difference between the projected restriction and the true one, i.e. the approximation error $\norm{[\Id-\proj{\spHX}]\TO_{\vert_{\spHX}}}$, needs to be small.

When $\spH$ is an infinite-dimensional {\em  universal} RKHS, both the above requirements are satisfied~\citep{Kostic2022}, i.e.~$\spH$ is dense in $\spF$ and the approximation error is zero, leading to an arbitrarily good approximation of dynamics with enough data. Still, another key issue arises: the norms on the a-priori chosen $\spH$ and the unknown $\spF$ do not coincide, since the latter depends on the process itself. This \textit{metric distortion} phenomenon has been recently identified as the source of spurious estimation of the spectra of $\TO$ ~\citep{kostic2023koopman}, limiting the utility of the learned transfer operators. Indeed, even if $\TO$ is self-adjoint, that is the eigenfunctions are orthogonal in $\spF$, the estimated ones will not be orthogonal in $\spH$, giving rise to spectral pollution~\citep{Kato76}. This motivates one to additionally require that iii) $\spH$ is a \textit{subspace} of $\spF$, i.e. both spaces have the same norm. 

To summarize, the desired optimal $\spH$ is the 
\textit{leading} \textit{invariant} 
\textit{subspace} of $\TO$, that is the subspace corresponding to the largest (in magnitude) eigenvalues of $\TO$. This subspace $\spH$ achieves zero approximation error, eliminates metric distortion and best approximates (in the dynamical system sense) the operator $\TO$. 
Since any RKHS $\spH$ is entirely described~by a {\em feature map}, learning a leading invariant subspace $\spH$ from data is, fundamentally, a \textit{representation learning} problem.

{\bf Approach~} We start by formalizing the problem of learning a good finite dimensional representation space for $\TO$, and then address the same for the generator $\LO$. We keep the discussion less formal here, and state our main results more precisely in the next section. Our approach is inspired by the following upper and lower bounds on the approximation error, a direct consequence of the norm change from $\spH$ to $\spF$, 
\begin{equation}\label{eq:bias_bound}
\norm{[\Id-\proj{\spHX}]\TO \proj{\spHX}}^2 \lambda_{\min}^+(C_{\spHX}) \leq \norm{[\Id-\proj{\spHX}]\TO_{\vert_{\spHX}}}^2 \leq \norm{[\Id-\proj{\spHX}]\TO \proj{\spHX}}^2 \lambda_{\max}(C_{\spHX}),    
\end{equation}
where 
$C_{\spHX}$ is the covariance operator on $\spHX$ w.r.t. the measure $\im$, while $\lambda_{\min}^+$ and $\lambda_{\max}$ are the smallest and largest non-null eigenvalues, respectively. 
Note that the norms on the hypothetical domain $\spH$ and true domain $\Lii$ coincide if and only if $C_{\spH}=I$, 
in which case equalities hold in \eqref{eq:bias_bound} and the approximation error becomes $\norm{[\Id-\proj{\spHX}]\TO \proj{\spHX}}$.

When the operator $\TO$ is known, the latter quantity can be directly minimized by standard numerical algorithms for spectral computation to find invariant subspaces \citep[see e.g.][]{golub2013matrix}.~Unfortunately, in our stochastic setting $\TO$ is unknown since we cannot compute the conditional expectation in \eqref{eq:transfer-operators}.
~To overcome this issue we propose a learning approach to recover the invariant space $\spHX$, which is rooted in the singular value decomposition, holding under the mild assumption that $\TO$ is a compact operator\footnote{This property is fulfilled by a large class of Markov processes \citep[see e.g.][]{Kostic2022} and is weaker than requiring the operator being Hilbert-Schmidt as in \citep{mardt2018vampnets}.}. The main idea is that the subspace made of the leading $r$ \textit{left singular functions} of $\TO$ serves as a good approximation of the desired leading invariant subspace of $\TO$. Namely, due to the orthonormality of the singular functions, we have that $C_{\spHX} = I$ and $\proj{\spHX} \TO$ becomes the $r$-truncated SVD of $\TO$, that is, its best rank-$r$ approximation. 
Therefore, according to \eqref{eq:bias_bound}, the approximation error is at most  
$\sval_{r+1}(\TO)$, which can be made arbitrarily small by rising $r$. Moreover, the distance of the subspace of left singular functions to the desired leading invariant subspace is determined by the "\textit{normality}" of $\TO$~\citep{TrefethenEmbree2020}. If the operator $\TO$ is \textit{normal}, that is $\TO \TO^* = \TO^* \TO$, then both its left and right singular spaces are invariant subspaces of $\TO$, resulting in zero approximation error irrespectively of $r$. This %
leads us to the following optimization problem 
\begin{equation}\label{eq:optimization}
\max_{\spHX,\spHY\subset\Lii}\left\{\hnorm{\proj{\spHX}\TO\proj{\spHY}}^2\;\vert\; C_{\spHX} = C_{\spHY} = I, \dim(\spHX)\leq r,\,\dim(\spHY)\leq r\right\}.
\end{equation}

{Relying on the application of Eckart-Young-Mirsky's Theorem~\citep{Kato76}, we can show that the desired representation space $\spHX$ can be computed by solving \eqref{eq:optimization}. Note that, in general, the auxiliary space $\spHY$ is needed to capture right singular functions, while if we have prior knowledge that $\TO$ is normal without loss of generality one can set $\spHX=\spHY$ in \eqref{eq:optimization}. 

The same reasoning cannot be applied straight away to the generator $\LO$ of a continuous dynamical system, which in is not even guaranteed to be bounded~\citep{Lasota1994}, let alone compact. For {\em time-reversal invariant} processes, however, 
$\TO$ and $\LO$ are self-adjoint, that is $\TO = \TO^*$ and $\LO = \LO^*$. This includes the important case of Langevin dynamics, 
which are of paramount importance for molecular dynamics 
\citep[see, e.g.,][]{tuckerman2010statistical}. For such systems we show (see Theorem~\ref{thm:main_cont}) that the leading $r$-dimensional subspace of the generator $\LO$ can be found by solving the following optimization problem {featuring the \emph{partial trace} objective}
\begin{equation}\label{eq:optimization_continious}
\max_{\spH\subset\Wii}\left\{\tr\left(\proj{\spH}\, \LO\,\proj{\spH}\right)\;\vert\; C_{\spH} = I, \dim(\spHX)\leq r\right\},
\end{equation}
where now the projector is $\proj{\spHX}\colon\Wii\to\Wii$. We stress that, since we assume $\LO=\LO^*$, we can relax the above assumption on the compactness of the transfer operators $(\TO_t)_{t\geq0}$, and still show that the leading invariant subspace of $\LO$ is the optimal $\spH$ in \eqref{eq:optimization_continious}. So, on such $\spH$ we can estimate well $\LO$ via generator operator regression, see e.g. \citep{hou2023sparse, klus2020generator}.

\section{Deep projection score functionals}\label{sec:score}
\vspace{-.1truecm} In this section we show how to solve the representation learning problems \eqref{eq:optimization} and \eqref{eq:optimization_continious} using DNNs. In the discrete case, we consider a generalized version of 
problem \eqref{eq:optimization}, which encompasses non-stationary processes, for which the probability distributions change along the trajectory. Namely, let $\rvX$ and $\rvY$ be two $\spX$-valued random variables with probability measures $\mX$ and $\mY$, respectively. Because of time-homogeneity, w.l.o.g. $\rvX$ models the state at time $0$ and $\rvY$ its evolution after some time $t$. Then, the transfer operator can be defined on the data-dependent domains $\TO_{t}\colon \LiiY\to\LiiX$. Replacing $\proj{\spHY}$ with $\proj{\spHY}'\colon \LiiY\to\LiiY$ in \eqref{eq:optimization}, and using the covariances $C_{\spHX}$ and $C_{\spHY}$ w.r.t. the measures $\mX$ and $\mY$, we obtain the appropriate representation learning problem for non-stationary process; see~App.~\ref{app:prob} for more details. Within this general setting 
we optimize two feature maps 
\begin{equation}
\embXp{}: \spX \to \R^r,\,\text{ and }\; \embYp{}:\spX \to \R^r,
 \label{eq:FM}
\end{equation}
parameterized by $w$ taking values in some set $\spParam$. 
Next, defining the two RKHSs\footnote{Here, $\embXp{,j}$ and $\embYp{,j}$ are the $j$-th component of $\embXp{}$ and $\embYp{}$, respectively.} $\spHXp:=\Span(\embXp{,j})_{j\in [r]}$ and $\spHYp:=\Span(\embYp{,j})_{j\in [r]}$, both equipped with the standard inner product, 
we aim to solve \eqref{eq:optimization} by optimizing over the weights $w$. To avoid solving the constrained optimization problem, we further propose to  relax the hard constraints in \eqref{eq:optimization} through a \textit{metric distortion loss} $\Reg\colon \R^{r\times r}\to\R_{+}$ that is zero if and only if it is applied to the identity matrix. Our choice in Section \ref{sec:exp}, well defined for SPD matrices $C$, is 
\begin{equation*}
    \Reg(C) :=  \tr(C^{2} - C -\ln(C)).
\end{equation*}
Finally, in Lem.~\ref{lm:proj_to_cov},~App. \ref{app:score} we show that \eqref{eq:optimization} can be solved by maximizing, for $\reg > 0$
\begin{equation}\label{eq:score}
\score^\reg(\param) =
\hnorm{(\Cxp)^{\dagger/2}\Cxyp (\Cyp)^{\dagger/2}}^2 - \reg \left(\Reg(\Cxp) + \Reg(\Cyp)\right),
\end{equation}
where we introduced the uncentered covariance matrices
\begin{equation}\label{eq:cov_crosscov}
\Cxp := \EE\, \embXp{}(\rvX) \embXp{}(\rvX)^\top,\; \Cyp := \EE \,\embYp{}(\rvY) \embYp{}(\rvY)^\top\,\text{ and }\, \Cxyp := \EE\, \embXp{}(\rvX) \embYp{}(\rvY)^\top.
\end{equation}
Notice that
if $\reg =0$ and} the covariances $\Cxp$ and $\Cyp$ are nonsingular, then the score reduces to $\hnorm{(\Cxp)^{-\frac{1}{2}}\Cxyp (\Cyp)^{-\frac{1}{2}}}^2$ which is called VAMP-2 score in the VAMPNets approach of~\citep{mardt2018vampnets}. Moreover, if in \eqref{eq:score} the HS norm is replaced by the nuclear norm, the score becomes 
the objective of canonical correlation analysis (CCA) \citep{harold1936,hardoon2004} in feature space. When DNNs are used to parametrize~\eqref{eq:FM}, such score is the objective of Deep-CCA~\citep{andrew2013deep}, known as VAMP-1 score \citep{mardt2018vampnets} in the context of molecular kinetics. 
Therefore, by maximizing $\score^\reg$ we look for the strongest \emph{linear correlation} between $\spHXp$ and $\spHYp$. We stress that differently from Deep-CCA and VAMPNets the crucial addition of $\Reg$ is to \emph{decorrelate} the features within each space, guiding the maximization towards a solution in which the norms of $\spHX$ and $\spHY$ coincide with those of $\LiiX$ and $\LiiY$, overcoming metric distortion. 

While in the optimal representation the covariances $\Cxp$ and $\Cyp$ are the identity, in general they are non-invertible, making the score $\score^\reg$ non-differentiable during optimization. Indeed, unless the rank of both covariances is stable for every $\param \in \spParam$, we might have exploding gradients~\citep{Golub1973}. 
Even ignoring differentiability issues, the use of the pseudo-inverse, as well as the use of the inverse in the non-singular case, can introduce severe numerical instabilities when estimating $\score^\reg$ and its gradients during the training process.~More precisely, {the numerical conditioning of evaluating $\score^\reg$ using \eqref{eq:score}} is determined by the covariance condition numbers 
$\eval_1(\Cxp) / \eval_r(\Cxp)$ and $\eval_1(\Cyp) / \eval_r(\Cyp)$ that can be very large in practical situations (see e.g. the fluid dynamics example in Sec.~\ref{sec:exp})
To overcome this issue, we introduce the relaxed score 
{
\begin{equation}\label{eq:rel_score}
\rscore^\reg(\param):= 
\hnorm{\Cxyp}^2\,/\, (\norm{\Cxp}\norm{\Cyp})
- \reg \left(\Reg(\Cxp) + \Reg(\Cyp)\right),
\end{equation}
which,} as implied by Theorem \ref{thm:main} below, is a lower bound for the score in \eqref{eq:score}.
A key advantage of this approach is that the score $\rscore^\gamma$ is both differentiable (apart from the trivial case $\Cxp$, $\Cyp=0$) and has stable gradients, since we avoid matrix inversion. Indeed, computing $\rscore^\reg$ is always well-conditioned, as the numerical operator norm has conditioning equal to one. 
The following theorem provides the theoretical backbone of our representation learning approach.
\begin{restatable}{theorem}{thmScore}\label{thm:main}
If $\TO{\colon}\LiiY\,{\to}\,\LiiX$ is compact, $\spHXp\,{\subseteq}\,\LiiX$ and $\spHYp\,{\subseteq}\,\LiiY$, then for all $\reg\geq0$  
\begin{equation}\label{eq:score_bound}
\rscore^{\reg}(\param)\leq\score^{\reg}(\param) \leq %
\sigma^2_1(\TO)+\dots+\sigma^2_r
(\TO).
\end{equation}
Moreover, if $(\embXp{,j})_{j\in[r]}$ and $(\embYp{,j})_{j\in[r]}$ are the leading $r$ left and right singular functions of $\TO$, respectively, then both equalities in \eqref{eq:score_bound} hold. Finally, if the operator $\TO$ is Hilbert-Schmidt, $\sval_{r}(\TO)>\sval_{r+1}(\TO)$ and $\reg>0$, then the ``only if'' relation is satisfied up to unitary equivalence.
\end{restatable}
This result establishes that the relaxed score is a tight lower bound, in the optimization sense, for \eqref{eq:score}. %
More precisely, in general $\abs{\rscore^\reg(w){-} \score^\reg(w)}$ is arbitrarily small for small enough $\Reg(w)$, while on the feasible set in \eqref{eq:optimization} both scores coincide. %
For instance, for a deterministic linear system this happens as soon as the classes of features $\embXp{}$ and $\embYp{}$ include linear functions. In general, we can appeal to the universal approximation properties of DNN \citep{cybenko} to %
fulfill the hypotheses of Theorem \ref{thm:main}. 

{\bf Generator learning~} We now consider how to learn the representation of continuous time-homogeneous processes arising from stochastic differential equations (SDE) of the form
\begin{equation}\label{eq:ito_SDE}
dX_t = A(X_t)\,dt + B(X_t)\,dW_t,
\end{equation}
where $A$ and $B$ are the drift and diffusion terms, respectively, and $W_t$ is a Wiener process.
The solution of \eqref{eq:ito_SDE} is a stochastic process that is time-reversal invariant w.r.t. its invariant distribution $\im$, \citep{arnold1974}. %
To show that solving \eqref{eq:optimization_continious} leads to learning the leading invariant subspace of $\LO$, we restrict to the case $\LO=\LO^*$ and consider the partial trace of $\LO$ w.r.t. to the subspace $\spHXp:=\Span(\embXp{,j})_{j\in[r]}\subseteq\Wii$ as our objective function. 
Being $\LO$ self-adjoint, we can relax the compactness assumption of Thm.~\ref{thm:main}, requiring only the much weaker condition that $\LO$ has the largest eigenvalues separated from the \textit{essential spectrum}~\citep{Kato76}. In particular, it holds $\tr\left(\proj{\spHXp}\, \LO\,\proj{\spHXp}\right)= \tr\left((\Cxp)^{\dagger}\Cxyd \right)$, where $\Cxyd =  \EE [\embXp{}(X)\, d\embXp{}(X)^\top]$ is the continuous version of the cross-covariance and $d\embXp{}(\cdot)$ is given by the It\={o} formula (see e.g. \citep{arnold1974,klus2020generator, hou2023sparse})
\begin{align}
d\embXp{,i}(x)  :=& \EE[\nabla\embXp{,i}(X)^\top dX/dt + \tfrac{1}{2}(dX/dt)^\top \nabla^2\embXp{,i}(X) (dX/dt)\,\vert\, X= x] \label{eq:ito_1} \\
=& \nabla\embXp{,i}(x)^\top A(X) + \tfrac{1}{2}\tr(B(x)^\top \nabla^2\embXp{,i}(x) B(x)). \label{eq:ito_2}
\end{align}
The following result, the proof of which is given in~App.~\ref{app:SDE}, then justifies our continuous DPNets.
\begin{restatable}{theorem}{thmScoreCont}\label{thm:main_cont}
If $\spHXp\subseteq\Wii$, and $\eval_1(\LO)\geq \cdots \geq \eval_{r+1}(\LO)$ are eigenvalues of $\LO$ above its essential spectrum,
then for every $\reg\geq0$ it holds%
\begin{equation}\label{eq:score_bound_cont}
\score_{\partial}^\reg(\param):= \tr\left((\Cxp)^{\dagger}\Cxyd \right) - \reg \Reg(\Cxp) \leq \eval_1(\LO)+\dots+\eval_r(\LO),
\end{equation}
and the equality is achieved  when $\embXp{,j}$ is the eigenfunction of $\LO$ corresponding to $\eval_j(\LO)$,  for 
$j\in[r]$. 
\end{restatable} 
Notice how $\tr\left((\Cxp)^{\dagger}\Cxyd \right)$ in Eq. \eqref{eq:score_bound_cont} is the sum of the (finite) eigenvalues of the symmetric eigenvalue problem $\Cxyd - \lambda \Cxp$. Thus, in non-pathological cases, its value and gradient can be computed efficiently in a numerically stable way, see e.g. \citep{Andrew1998}.
\section{Methods}\label{sec:methods}
\vspace{-.1truecm}
{\bf Learning the representation~} To optimize a DPNet representation one needs to replace the population covariances and cross-covariance in Eqs~\eqref{eq:score},~\eqref{eq:rel_score} and \eqref{eq:score_bound_cont} with their empirical counterparts. In practice, given $\embXp{}, \embYp{} \colon\spX \to \R^{r}$ and a dataset $\Data=(x_i,x'_i)_{i\in[n]}$ consisting of samples from the joint  distribution of $(X,X')$ one estimates the needed covariances $\Cxp$, $\Cyp$ and $\Cxyp$ as
\begin{equation*}\label{eq:em_cov}
\ECxp \!=\! \tfrac{1}{n} 
{\sum_{i\in[n]}} \embXp{}(x_i) \embXp{}(x_i)^\top\!, \ECyp \!=\! \tfrac{1}{n} 
{\sum_{i\in[n]}} \embYp{}(x'_i) \embYp{}(x'_i)^\top
\text{ and }\,
\ECxyp\! =\! \tfrac{1}{n} 
{\sum_{i\in[n]}} \embXp{}(x_i) \embYp{}(x'_i)^\top\!\!.
\end{equation*}

\begin{algorithm}[t!]
\caption{DPNets Training %
}
\label{alg:dpnets}
\begin{algorithmic}[1]
\Require Data $\Data$; metric loss parameter $\reg$; DNNs $\psi_w, \psi'_w: \spX \to \R^r$; optimizer $U$; \# of steps $k$.
\State Initialize DNN weights to $w_1$.
\For{$j=1$ to $k$}; 
    \State  $\hat{C}_{X}^{w_{j}}, \hat{C}_{X^{\prime}}^{w_{j}}, \hat{C}_{XX^{\prime}}^{w_{j}} \gets$ %
     Covariances for $\psi_{w_{j}}$ and $\psi'_{w_j}$ from a mini-batch of $m \leq n$ samples.
    \State $\mathrlap{F(w_{j})}\phantom{\hat{C}_{X}^{w_{j}}, \hat{C}_{X^{\prime}}^{w_{j}}, \hat{C}_{XX^{\prime}}^{w_{j}}} \gets -  \erscore^\reg_m(\param_{j}) $ from ~\eqref{eq:rel_score} and the covariances in Step 3.
    \State $\mathrlap{w_{j+1}}\phantom{\hat{C}_{X}^{w_{k}}, \hat{C}_{X^{\prime}}^{w_{k}}, \hat{C}_{XX^{\prime}}^{w_{j}}}  \gets U(w_{j}, \nabla F(w_{j})$) where $\nabla F(w_{j})$ is computed using backpropagation.

\EndFor
\State \textbf{return} representations $\psi_{w_K}$, $\psi'_{w_K}$.
\end{algorithmic}
\end{algorithm}

\vspace{-.2truecm}
In~App.~\ref{app:concScore} we prove that the empirical scores $\escore^{\reg}_{n}\colon \spParam\to\R$ and $\erscore^{\reg}_{n} \colon \spParam\to\R$, that is the scores computed using empirical covariances, concentrate around the true ones with high probability for any fixed $w$. Obtaining uniform guarantees over $w$ requires advanced tools from empirical processes and regularity assumptions on the representations which could be the subject of future work. The difficulties arising from the estimation of $\score^{\reg}$ during optimization are also addressed.%

In Alg.~\ref{alg:dpnets} we report the training procedure for DPNets-relaxed (using score \eqref{eq:rel_score}), which can also be applied to~\eqref{eq:score}  and~\eqref{eq:score_bound_cont} suitably modifying step 4. When using mini-batch SGD methods to compute $\nabla \escore^\reg_{m} (w)$ or $\nabla \erscore^\reg_m(w)$, the batch size $m$ should be taken sufficiently large to mitigate the impact of biased stochastic estimators. The time complexity of the training algorithm is addressed in App.~\ref{app:train}.
 
\vspace{-.1truecm}
{\bf Operator regression~} Once a representation $\embXp{}$ is learned, it can be used %
within the framework of operator regression, see~Fig.~\ref{fig:pipeline}. 
Following the reasoning in \citep{Kostic2022} one sees that any model $\ETO_{\param} \colon \spHXp\to \spHXp$ of the transfer operator $\TO$
acts on functions in $\spH_{\param}$ as
\begin{equation}\label{eq:TO_estimator}
h_z:=\embXp{}(\cdot)^\top z\quad\mapsto \quad\ETO_{\param} h_z:= \embXp{}(\cdot)^\top \EEstim z,\;z\in\R^{r},    
\end{equation}
where $\EEstim\in\R^{r\times r}$ is a matrix in the representation space. 
For example, denoting the data matrices 
$\embMXp = [\embXp{}(x_1)\,\vert\,\cdots\,\vert\,\embXp{}(x_n)],\,\embMYp = [\embXp{}(x_1')\,\vert\,\cdots\,\vert\,\embXp{}(x_n')]\in\R^{r\times n}$, the ordinary least square estimator (OLS) minimizes the empirical risk $\hnorm{\embMYp - \EEstim^\top\,\embMXp}^2$, yielding $\EEstim:= (\embMXpT)^\dagger \embMYpT$.

\vspace{-.15truecm}
{\bf Downstream tasks~} Once $\EEstim$ is obtained, it can be used predict the next expected state of the dynamics given the initial one, compute the modal decomposition of an 
observable~\citep{Kutz2016}, estimate the spectrum of the transfer operator~\citep{kostic2023koopman} {or controlling the dynamics~\citep{Proctor2016}}.  
Indeed, recalling that $\rvY$ is the one step ahead evolution of $\rvX$, we can use \eqref{eq:TO_estimator} to approximate $\EE[ h_z(\rvY)\,\vert\,\rvX = x]$ as $ \embXp{}(x)^\top \EEstim z$. Moreover, relying on the reproducing property of $\spHXp$, the predictions can be extended to functions $f\colon\spX\to\R^{\ell}$ as 
$\EE[ f(\rvY)\,\vert\,\rvX = x] \approx  \embXp{}(x)^\top (\embMXpT)^\dagger \widehat{F}'$, where $\widehat F' = [f(x_1')\,\vert\,\ldots\,\vert\,f(x_n')]^\top\in\R^{n\times \ell}$ is the matrix of the observations of the evolved data. Clearly, when the observable is the state itself (i.e. $f(x)=x$), we obtain one step ahead predictions. 

\vspace{-.1truecm}
Next, observe that eigenvalue decomposition of the matrix $\EEstim$ leads to the eigenvalue decomposition of the operator $\ETO$. 
Namely, let $(\eeval_i, \elevec_i, \erevec_i)\in\C \times \C^d \times \C^d$ be an eigen-triplet made of eigenvalue, left eigenvector and right eigenvector of $\EEstim$, that is $\EEstim \erevec_i = \eeval_i \erevec_i$, $ \elevec_i^* \EEstim = \eeval_i \elevec_i^*$ and $\elevec_i^*\erevec_{k} = \delta_{i,k}$, $i,k\in[r]$, we directly obtain, using \eqref{eq:TO_estimator}, the spectral decomposition
\begin{equation}\label{eq:estimator_evd}
\ETO_{\param}= \textstyle{\sum_{i\in[r]}} \eeval_i \,\erefun_i \otimes \elefun_i, \quad\text{ where }\quad \erefun_i(x) := \embXp{}(x)^\top \erevec_i\;\text{ and }\; \elefun_i(x) :=  (\elevec_i)^*\embXp{}(x).
\end{equation}
\vspace{-.4truecm}

Finally, we can use the spectral decomposition of $\EEstim$ to efficiently forecast observables for several time-steps in the future via $\EE[h_z(X_t)\,\vert\,X_0=x]\approx \sum_{i\in[r]} \eeval_i^t \,\scalarp{\elefun_i,h_z}_{\spHXp}\,\erefun_i(x)$, which is known as the extended dynamic mode decomposition (EDMD), see e.g. ~\citep{Brunton2022}. Noticing that $\scalarp{\elefun_i,h_z}_{\spHXp} = (\elevec_i)^* z = (\elevec_i)^*\EEstim z / \eeval_i$, when $\eeval_i\neq0$, and, again, using the reproducing property, we can extend this approximation to vector-valued observables $f\colon\spX\to\R^{\ell}$ as  
\begin{equation}\label{eq:KMD}
\EE[f(X_t)\,\vert\,X_0=x]\approx \textstyle{\sum_{i\in[r]}} \eeval_i^{t-1} \,\big(\embXp{}(x)^\top \erevec_i\big)\,\,\big(\elevec_i^*(\embMXpT)^\dagger \widehat{F}'\big)\in\R^{\ell}.    
\end{equation}

\section{Experiments}\label{sec:exp}
\vspace{-.1truecm}
In this experimental section we show that DPNets (i) learn reliable representations of non-normal dynamical systems, (ii) can be used in its continuous-dynamics form~\eqref{eq:score_bound_cont} to approximate the eigenvalues of a Langevin generator, and (iii) outperforms kernel methods and auto-encoders in large-scale and/or structured (e.g. graphs and images) data settings. To have a fair comparison, every neural network model in these experiments has been trained on the same data splits, batch sizes, number of epochs, architectures and seeds. The learning rate, however, has been optimized for each one separately. We defer every technical detail, as well as additional results to App.~\ref{app:exp}. The code to reproduce the examples can be found at \url{https://pietronvll.github.io/DPNets/}, and it heavily depends on Kooplearn \url{https://kooplearn.readthedocs.io/}. 

\textbf{Baselines}  We compared our methods DPNets ($\score^{\reg}$) and DPNets-relaxed ($\rscore^{\reg}$), where appropriate, with the following baselines:  Dynamic AE (DAE) of \citep{Lusch2018}, Consistent AE of \citep{azencot2020forecasting}, DMD of \citep{Schmid2010}, KernelDMD and ExtendedDMD of \citep{Williams2015a, Williams2015b}, and VAMPNets of \citep{mardt2018vampnets}.

{\bf Logistic Map~} We study the 
dynamical system
$X_{t + 1} = (4X_{t}(1 - X_{t}) + \xi_{t})~{\rm mod}~ 1$, for $\spX = [0, 1)$ and $\xi_{t}$ being i.i.d. trigonometric noise~\citep{Ostruszka2000}
. 
The associated transfer operator $\TO$ is {\em non-normal}, making the learning of its spectral decomposition particularly challenging~\citep{kostic2023koopman}.
Since $\TO$ can be computed exactly, we can
sidestep the problem of operator regression (see ~Fig.~\ref{fig:pipeline}), and focus directly on evaluating the quality of the representation encoded by $\spH_{\param}$. We evaluate DPNets on two different metrics: (i) the optimality gap $\sum_{i=1}^{r} \sigma_{i}^{2}(\TO) \,{-}\, \score^{0}(\param)$ 
for $r = 3$ and (ii) the spectral error, given by  $\max_{i}\min_{j} |\lambda_{i}(\proj{\spHX} \TO_{\vert_{\spHX}}) - \lambda_{j}(\TO)|$. While (i) informs on how close one is to solve~\eqref{eq:optimization}, (ii) measures how well the true eigenvalues of $\TO$ can be recovered within the representation space $\spH_{\param}$. In Tab.~\ref{tab:1} we compare the DPNets representation against VAMPNets, ExtendedDMD with a feature map of Chebyshev polynomials and the feature defining the trigonometric noise~\cite{Ostruszka2000} of the process itself. Notice that in this challenging setting recovering eigenvalues via singular vectors is generally problematic, so a moderate optimality gap may lead to a larger spectral error. In Fig.~\ref{fig:eigvals} we report the evolution of the spectral error during training for DPNets and VAMPNets, while we defer to App.~\ref{app:exp_logistic} an in-depth analysis of the role of the feature dimension. Notice that DPNets and DPNets-relaxed excel in both metrics.

\riccardo{In Fig2 the DPNets curve does not look good (converges and then diverges). Any idea why? Is it overfitting?
PN: Yes, possibly overfitting. MP: should we comment on this?
RG: If overfitting it seems that either early stopping or more data is needed. Also, i would expect the optimality gap for such simple problem to go closer to zero with enough data. If left like this i would not comment but It's a bit concerning.
}
\riccardo{noticed now that curves are for spectral error, i like optimzality gap more, is there overfitting there? maybe we can put that curve in appendix eventually.}

\vspace{-.1truecm}
{\bf Continuous dynamics~} %
We investigate a one-dimensional {\em continuous} SDE describing the stochastic motion of a particle into the Schwantes potential~\citep{Schwantes2015}. %
The invariant distribution for this process is the Boltzmann distribution $\pi(dx) \propto e^{-\beta V(x)}dx$, where $V(x)$~ is the potential at state $x$. 
The non-null eigenvalues of $\LO$ hold physical significance, as their absolute value represents the average rate at which particles cross one of the system's potential barriers~\citep{Kramers1940}; our objective is to accurately estimate them. Here $X_{t}$ is sampled \emph{non-uniformly} according to a geometric law.
In the lower panel of Fig.~\ref{fig:eigvals} we report the estimated  transition rates, of $\LO$ along the DPNets training loop. Notice that the embedding $\embXp{}$ progressively improves eigenvalue estimation, indicating that the invariant subspaces of the generator $\LO$ are well captured.

\begin{table}[t]
\setlength{\tabcolsep}{5pt}
\small
\centering
\begin{minipage}[t]{0.49\textwidth}
\begin{tabular}{r|cc}

\toprule
          Representation &        Spectral Error         &        Optimality Gap         \\
\midrule
        \textbf{DPNets} &          0.28 (0.03)          & \textbf{0.64} (\textbf{0.01}) \\
 \textbf{DPNets-relaxed} & \textbf{0.06} (\textbf{0.05}) &          1.19 (0.04)          \\
                VAMPNets &          0.21 (0.09)          &          0.97 (0.22)          \\
                 Cheby-T &             0.20              &             1.24              \\
             NoiseKernel &             0.19              &             2.17              \\
\bottomrule
\end{tabular}
\vspace{-.2truecm}
\caption{{\em Logistic Map}. Comparison of DPNets to relevant baselines. Mean and standard deviation are over 20 independent runs. We used $r=7$.}\label{tab:1}
\end{minipage}
\hfill
\begin{minipage}[t]{0.49\textwidth}
\begin{tabular}{r|ccc}
\toprule
          Model &  $\mathcal{P}$  &  Transition  &   Enthalpy $\Delta H$   \\
\midrule
 \textbf{DPNets} &      \textbf{12.84}      &     \textbf{17.59 ns}      & \textbf{-1.97 kcal/mol} \\
    Nys-PCR &      7.02       &      5.27 ns      & -1.76 kcal/mol \\
    Nys-RRR &      2.22       &      0.89 ns      & -1.44 kcal/mol \\
     \hline 
     \\
      Reference &        -        &       40 ns       & -6.1 kcal/mol  \\
\bottomrule
\end{tabular}
\vspace{-.2truecm}
\caption{{\em Chignolin}. Comparison between DPNets and kernel methods with Nystr\"om sampling \citep{Meanti2023}.}
    \label{tab:chignolin}
\end{minipage}
\vspace{-.7truecm}
\end{table}
\vspace{-.1truecm}
{\bf Ordered MNIST~} Following~\cite{Kostic2022}, we create a stochastic dynamical system by randomly sampling images from the MNIST dataset according to the rule that $X_t$ should be an image of the digit $t$ $({\rm mod} ~5)$ for all $t \in \N_0$.
Given an image from the dataset with label $c$, a model for the transfer operator $\TO$ of this system should then be able to produce an MNIST-alike image of the next digit in the cycle. 
 In the upper panel of Fig.~\ref{subfig:forecasts} we thus evaluate DPNets and a number of baselines by how accurate is an ``oracle'' supervised MNIST classifier (test accuracy for in-distribution $\geq 99\%$) in predicting the correct label $c + t ~({\rm mod}~ 5)$ after $t$ steps of evolution. DPNets consistently retain an accuracy above $95\%$, while for every other method it degrades. The ``reference'' line corresponds to random guessing, while the ``Oracle-Feature'' baseline is an operator regression model (EDMD) using, as the dictionary of functions, the output logits of the oracle, and despite having been trained with the true labels, its performance degrades drastically after $t \geq 5$.

\vspace{-.1truecm}
{\bf Fluid dynamics~} %
We study the classical problem of the transport of a passive scalar field by a 2D fluid flow past a cylinder~\citep{Raissi2020}. 
Each data point comprises a regular 2D grid that encompasses fluid variables, including velocity, pressure, and scalar field concentration at each grid point. 
This system is non-stationary and is also known to exhibit non-normal dynamics~\citep[see][]{TrefethenEmbree2020}. %
We evaluate each trained baseline by feeding to it the last snapshot of the train trajectory and evaluating the relative RMSE (that is, the RMSE normalized by the data variance) between the forecasts and the subsequent test snapshots. In this experiment, we also have a {\em physics-informed} (PINN) baseline not related to transfer operator learning, which is however the model for which this dataset was created in the first place.
Remarkably, the forecasting error of DPNets does not grow sensibly with time as it does for every other method.

\begin{figure}[t!]
  \centering
  \includegraphics[width=\textwidth]{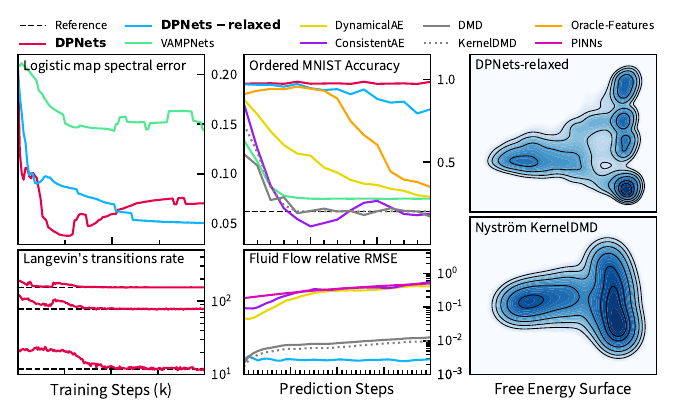}
  \begin{minipage}[t]{0.298\textwidth}
    \vspace{-0.55cm}
    \centering
    \caption{{\em Eigenvalue error decays during training}. Upper: spectral error for the logistic map. Lower: DPNets-estimated eigenvalues of $\LO$ for the Langevin dynamics. }
    \label{fig:eigvals}
  \end{minipage}
  \hspace{0.43cm}
  \begin{minipage}[t]{0.298\textwidth}
    \vspace{-0.55cm}
    \centering
    \caption{{\em Forecasting with DPNets}.~Upper:~classification accuracy over time for ordered MNIST. Lower: forecasting RMSE over time for fluid dynamic example.}
    \label{subfig:forecasts}
  \end{minipage}
  \hspace{0.43cm}
  \begin{minipage}[t]{0.298\textwidth}
    \vspace{-0.58cm}
    \centering
    \caption{Free energy surface of the 2 slowest modes of Chignolin,estimated by DPNets and Nyst\"rom PCR. To be compared with~\citet{Bonati2021}.
    }
    \label{fig:chignolin}
  \end{minipage}
  \vspace{-.9truecm}
\end{figure}

\vspace{-.1truecm}
{\bf The Metastable states of Chignolin~} In our last experiment, we study the dynamics of Chignolin, a folding protein, from a $106 \, \mu{\rm s}$ long molecular dynamics simulation sampled every $200 \, {\rm ps}$, totalling over $500,000$ data points~\citep{LindorffLarsen2011}. We focus on the leading eigenfunctions of the transfer operator, which are known~\citep{Schutte2001} to provide a simple characterization of the slowest (and usually most important) physical processes occurring along the dynamics. From the leading left eigenfunctions of $\ETO$, one can indeed construct the free energy surface (see~Fig.~\ref{fig:chignolin}), whose local minima identify the metastable states of the system. The free energy surface, indeed, is proportional to the negative $\log$-pdf of finding the system in a given state while at thermodynamic equilibrium, meaning that a state of {\em low} free energy is highly probable, hence metastable. For Chignolin, the metastable states have been thoroughly studied~\citep[see e.g.][]{Novelli2022}. The leading left eigenfunction of $\ETO$ encodes the folding-unfolding transition and the typical transition time is the {\em implied timescale}~\citep{Mardt2019} 
associated with its eigenvalue. The difference between the local minima of the free energy surface encodes the enthalpy $\Delta H$ of the transition. In ~Tab.~\ref{tab:chignolin} we compare these quantities to the reference values reported in~\citet{LindorffLarsen2011}. We trained a GNN-based DPNet-relaxed, as both DPNets unrelaxed and VAMPNets failed to converge, possibly due to the large scale of the data. We compared it to a KernelDMD estimator trained with the recent Nystr\"om sketching technique~\citep{Meanti2023} as classical kernel methods are completely intractable at this scale\footnote{A back-of-the-envelope calculation shows that 450 GBs would be needed just to store kernel matrices in single precision.}. Notice that DPNets succeed in finding additional meta-stable states of the system, which match the analysis of \citep{Bonati2021}; see~App.~\ref{app:exp} for more discussion.

\section{Conclusions}
\vspace{-.1truecm}
We propose a %
framework for learning a representation of dynamical systems, based on orthogonal projections in data-spaces. It 
captures a leading invariant subspace of the transfer operator and can be applied to both discrete and continuous dynamical systems. In the discrete case, the representation is learned through the optimization of a smooth and numerically well conditioned objective function. 
Extensive numerical experiments demonstrate the effectiveness and generality of DPNets in various settings, suggesting that they are a promising tool for data-driven dynamical systems. A limitation of this work is that the score functional for the continuous systems might be unstable since it leverages covariance matrix inversion. Moreover, a future direction would be to study the statistical learning properties of the algorithm presented here.

\section*{Acknowledgements}
We acknowledge the financial support from the PNRR MUR Project PE000013 CUP J53C22003010006 "Future Artificial Intelligence Research (FAIR)", funded by the European Union – NextGenerationEU, EU Project ELIAS under grant agreement No. 101120237 and the chair "Business Analytic for Future Banking" sponsored by Natixis-Groupe BPCE.

\bibliographystyle{apalike}
{%
\bibliography{bibliography}

\begin{thebibliography}{}

\bibitem[Alexander and Giannakis, 2020]{Alexander2020}
Alexander, R. and Giannakis, D. (2020).
\newblock Operator-theoretic framework for forecasting nonlinear time series
  with kernel analog techniques.
\newblock {\em Physica D: Nonlinear Phenomena}, 409:132520.

\bibitem[Allen, 2007]{allen2007}
Allen, E. (2007).
\newblock {\em Modeling with It{\^o} stochastic differential equations},
  volume~22.
\newblock Springer Science \& Business Media.

\bibitem[Andrew and Tan, 1998]{Andrew1998}
Andrew, A.~L. and Tan, R. C.~E. (1998).
\newblock Computation of derivatives of repeated eigenvalues and the
  corresponding eigenvectors of symmetric matrix pencils.
\newblock {\em SIAM Journal on Matrix Analysis and Applications},
  20(1):78--100.

\bibitem[Andrew et~al., 2013]{andrew2013deep}
Andrew, G., Arora, R., Bilmes, J., and Livescu, K. (2013).
\newblock Deep canonical correlation analysis.
\newblock In {\em International Conference on Machine Learning}, pages
  1247--1255. PMLR.

\bibitem[Arnold, 1974]{arnold1974}
Arnold, L. (1974).
\newblock {\em Stochastic Differential Equations: Theory and Applications},
  volume~2.
\newblock John Wiley \& Sons.

\bibitem[Aronszajn, 1950]{aron1950}
Aronszajn, N. (1950).
\newblock Theory of reproducing kernels.
\newblock {\em Transactions of the American Mathematical Society},
  68(3):337--404.

\bibitem[Azencot et~al., 2020]{azencot2020forecasting}
Azencot, O., Erichson, N.~B., Lin, V., and Mahoney, M. (2020).
\newblock Forecasting sequential data using consistent koopman autoencoders.
\newblock In {\em International Conference on Machine Learning}, pages
  475--485. PMLR.

\bibitem[Bevanda et~al., 2021]{Bevanda2021}
Bevanda, P., Beier, M., Kerz, S., Lederer, A., Sosnowski, S., and Hirche, S.
  (2021).
\newblock Koopmanizing{F}lows: {D}iffeomorphically {L}earning {S}table
  {K}oopman {O}perators.
\newblock {\em arXiv preprint arXiv.2112.04085}.

\bibitem[Bonati et~al., 2021]{Bonati2021}
Bonati, L., Piccini, G., and Parrinello, M. (2021).
\newblock Deep learning the slow modes for rare events sampling.
\newblock {\em Proceedings of the National Academy of Sciences}, 118(44).

\bibitem[Bouvrie and Hamzi, 2017]{Bouvrie2017}
Bouvrie, J. and Hamzi, B. (2017).
\newblock {K}ernel {M}ethods for the {A}pproximation of {N}onlinear {S}ystems.
\newblock {\em {SIAM} Journal on Control and Optimization}, 55(4):2460--2492.

\bibitem[Brunton et~al., 2022]{Brunton2022}
Brunton, S.~L., Budi{\v{s}}i{\'{c}}, M., Kaiser, E., and Kutz, J.~N. (2022).
\newblock Modern {K}oopman {T}heory for {D}ynamical {S}ystems.
\newblock {\em {SIAM} Review}, 64(2):229--340.

\bibitem[Cannarsa et~al., 2020]{cannarsa2020mathematical}
Cannarsa, P., Mansutti, D., and Provenzale, A. (2020).
\newblock {\em Mathematical Approach to Climate Change and Its Impacts: MAC2I},
  volume~38.
\newblock Springer.

\bibitem[Chanussot et~al., 2021]{Chanussot2021}
Chanussot, L., Das, A., Goyal, S., Lavril, T., Shuaibi, M., Riviere, M., Tran,
  K., Heras-Domingo, J., Ho, C., Hu, W., Palizhati, A., Sriram, A., Wood, B.,
  Yoon, J., Parikh, D., Zitnick, C.~L., and Ulissi, Z. (2021).
\newblock Open catalyst 2020 ({OC}20) dataset and community challenges.
\newblock {\em {ACS} Catalysis}, 11(10):6059--6072.

\bibitem[Cybenko, 1989]{cybenko}
Cybenko, G. (1989).
\newblock Approximation by superpositions of a sigmoidal function.
\newblock {\em Mathematics of control, signals and systems}, 2(4):303--314.

\bibitem[Das and Giannakis, 2020]{Das2020}
Das, S. and Giannakis, D. (2020).
\newblock Koopman spectra in reproducing kernel {H}ilbert spaces.
\newblock {\em Applied and Computational Harmonic Analysis}, 49(2):573--607.

\bibitem[Fan et~al., 2021]{Fan2021}
Fan, F., Yi, B., Rye, D., Shi, G., and Manchester, I.~R. (2021).
\newblock {L}earning {S}table {K}oopman {E}mbeddings.
\newblock {\em arXiv preprint arXiv.2110.06509}.

\bibitem[Federici et~al., 2024]{federici2024latent}
Federici, M., Forr{\'e}, P., Tomioka, R., and Veeling, B.~S. (2024).
\newblock Latent representation and simulation of markov processes via
  time-lagged information bottleneck.
\newblock In {\em The Twelfth International Conference on Learning
  Representations}.

\bibitem[Fisher et~al., 2009]{fisher2009data}
Fisher, M., Nocedal, J., Tr{\'e}molet, Y., and Wright, S.~J. (2009).
\newblock Data assimilation in weather forecasting: a case study in
  pde-constrained optimization.
\newblock {\em Optimization and Engineering}, 10(3):409--426.

\bibitem[Fukumizu et~al., 2007]{fukumizu2007}
Fukumizu, K., Bach, F.~R., and Gretton, A. (2007).
\newblock Statistical consistency of kernel canonical correlation analysis.
\newblock {\em Journal of Machine Learning Research}, 8(14):361--383.

\bibitem[Ghorbani et~al., 2022]{Ghorbani2022}
Ghorbani, M., Prasad, S., Klauda, J.~B., and Brooks, B.~R. (2022).
\newblock {GraphVAMPNet}, using graph neural networks and variational approach
  to markov processes for dynamical modeling of biomolecules.
\newblock {\em The Journal of Chemical Physics}, 156(18):184103.

\bibitem[Golub and Pereyra, 1973]{Golub1973}
Golub, G.~H. and Pereyra, V. (1973).
\newblock The differentiation of pseudo-inverses and nonlinear least squares
  problems whose variables separate.
\newblock {\em {SIAM} Journal on Numerical Analysis}, 10(2):413--432.

\bibitem[Golub and Van~Loan, 2013]{golub2013matrix}
Golub, G.~H. and Van~Loan, C.~F. (2013).
\newblock {\em Matrix Computations}.
\newblock JHU press.

\bibitem[Goodfellow et~al., 2016]{goodfellow2016deep}
Goodfellow, I., Bengio, Y., and Courville, A. (2016).
\newblock {\em Deep Learning}.
\newblock MIT press.

\bibitem[Gorini and Kossakowski, 1976]{Gorini1976}
Gorini, V. and Kossakowski, A. (1976).
\newblock Completely positive dynamical semigroups of {N}-level systems.
\newblock {\em Journal of Mathematical Physics}, 17(5):821.

\bibitem[Gretton et~al., 2005]{gretton2005measuring}
Gretton, A., Bousquet, O., Smola, A., and Sch{\"o}lkopf, B. (2005).
\newblock Measuring statistical dependence with hilbert-schmidt norms.
\newblock In {\em Algorithmic Learning Theory: 16th International Conference,
  ALT 2005, Singapore, October 8-11, 2005. Proceedings 16}, pages 63--77.
  Springer.

\bibitem[Hardoon et~al., 2004]{hardoon2004}
Hardoon, D.~R., Szedmak, S., and Shawe-Taylor, J. (2004).
\newblock Canonical correlation analysis: An overview with application to
  learning methods.
\newblock {\em Neural Computation}, 16(12):2639--2664.

\bibitem[Harold, 1936]{harold1936}
Harold, H. (1936).
\newblock Relations between two sets of variates.
\newblock {\em Biometrika}, 28(3/4):321.

\bibitem[Hou et~al., 2023]{hou2023sparse}
Hou, B., Sanjari, S., Dahlin, N., Bose, S., and Vaidya, U. (2023).
\newblock Sparse learning of dynamical systems in rkhs: An operator-theoretic
  approach.
\newblock {\em ICML2023}.

\bibitem[Kato, 1976]{Kato76}
Kato, T. (1976).
\newblock {\em {Perturbation theory for linear operators; 2nd ed.}}
\newblock Grundlehren der mathematischen Wissenschaften : a series of
  comprehensive studies in mathematics. Springer, Berlin.

\bibitem[Kawahara, 2016]{Kawahara2016}
Kawahara, Y. (2016).
\newblock {D}ynamic {M}ode {D}ecomposition with {R}eproducing {K}ernels for
  {K}oopman {S}pectral {A}nalysis.
\newblock In {\em Advances in Neural Information Processing Systems},
  volume~29.

\bibitem[Klus et~al., 2020]{klus2020generator}
Klus, S., N{\"u}ske, F., and Hamzi, B. (2020).
\newblock Kernel-based approximation of the koopman generator and
  schr{\"o}dinger operator.
\newblock {\em Entropy}, 22(7):722.

\bibitem[Klus et~al., 2019]{Klus2019}
Klus, S., Schuster, I., and Muandet, K. (2019).
\newblock Eigendecompositions of transfer operators in reproducing kernel
  {H}ilbert spaces.
\newblock {\em Journal of Nonlinear Science}, 30(1):283--315.

\bibitem[Korda and Mezi{\'{c}}, 2017]{Korda2017}
Korda, M. and Mezi{\'{c}}, I. (2017).
\newblock {O}n {C}onvergence of {E}xtended {D}ynamic {M}ode {D}ecomposition to
  the {K}oopman operator.
\newblock {\em Journal of Nonlinear Science}, 28(2):687--710.

\bibitem[Kostic et~al., 2023]{kostic2023koopman}
Kostic, V., Lounici, K., Novelli, P., and Pontil, M. (2023).
\newblock Sharp spectral rates for koopman operator learning.
\newblock {\em Advances in Neural Information Processing Systems}, 36.

\bibitem[Kostic et~al., 2022]{Kostic2022}
Kostic, V., Novelli, P., Maurer, A., Ciliberto, C., Rosasco, L., and Pontil, M.
  (2022).
\newblock Learning dynamical systems via {K}oopman operator regression in
  reproducing kernel hilbert spaces.
\newblock In {\em Advances in Neural Information Processing Systems}.

\bibitem[Kovachki et~al., 2023]{Kovachki2023}
Kovachki, N., Li, Z., Liu, B., Azizzadenesheli, K., Bhattacharya, K., Stuart,
  A., and Anandkumar, A. (2023).
\newblock Neural operator: Learning maps between function spaces with
  applications to pdes.
\newblock {\em Journal of Machine Learning Research}, 24(89):1--97.

\bibitem[Kramers, 1940]{Kramers1940}
Kramers, H. (1940).
\newblock Brownian motion in a field of force and the diffusion model of
  chemical reactions.
\newblock {\em Physica}, 7(4):284--304.

\bibitem[Kutz et~al., 2016]{Kutz2016}
Kutz, J.~N., Brunton, S.~L., Brunton, B.~W., and Proctor, J.~L. (2016).
\newblock {\em {D}ynamic {M}ode {D}ecomposition}.
\newblock Society for Industrial and Applied Mathematics.

\bibitem[Lasota and Mackey, 1994]{Lasota1994}
Lasota, A. and Mackey, M.~C. (1994).
\newblock {\em Chaos, Fractals, and Noise}, volume~97 of {\em Applied
  Mathematical Sciences}.
\newblock Springer New York.

\bibitem[Li et~al., 2017]{Li2017}
Li, Q., Dietrich, F., Bollt, E.~M., and Kevrekidis, I.~G. (2017).
\newblock Extended dynamic mode decomposition with dictionary learning: A
  data-driven adaptive spectral decomposition of the koopman operator.
\newblock {\em Chaos: An Interdisciplinary Journal of Nonlinear Science},
  27(10):103111.

\bibitem[Li et~al., 2022]{Li2022}
Li, Z., Meunier, D., Mollenhauer, M., and Gretton, A. (2022).
\newblock Optimal rates for regularized conditional mean embedding learning.
\newblock In {\em Advances in Neural Information Processing Systems}.

\bibitem[Lindblad, 1976]{Lindblad1976}
Lindblad, G. (1976).
\newblock On the generators of quantum dynamical semigroups.
\newblock {\em Communications in Mathematical Physics}, 48(2):119--130.

\bibitem[Lindorff-Larsen et~al., 2011]{LindorffLarsen2011}
Lindorff-Larsen, K., Piana, S., Dror, R.~O., and Shaw, D.~E. (2011).
\newblock How fast-folding proteins fold.
\newblock {\em Science}, 334(6055):517--520.

\bibitem[Lusch et~al., 2018]{Lusch2018}
Lusch, B., Kutz, J.~N., and Brunton, S.~L. (2018).
\newblock Deep learning for universal linear embeddings of nonlinear dynamics.
\newblock {\em Nature Communications}, 9(1).

\bibitem[Mardt et~al., 2019]{Mardt2019}
Mardt, A., Pasquali, L., No{\'e}, F., and Wu, H. (2019).
\newblock Deep learning markov and koopman models with physical constraints.
\newblock In {\em Mathematical and Scientific Machine Learning}.

\bibitem[Mardt et~al., 2018]{mardt2018vampnets}
Mardt, A., Pasquali, L., Wu, H., and No{\'e}, F. (2018).
\newblock Vampnets for deep learning of molecular kinetics.
\newblock {\em Nature Communications}, 9(1):5.

\bibitem[McCarty and Parrinello, 2017]{McCarty2017}
McCarty, J. and Parrinello, M. (2017).
\newblock A variational conformational dynamics approach to the selection of
  collective variables in metadynamics.
\newblock {\em The Journal of Chemical Physics}, 147(20):204109.

\bibitem[Meanti et~al., 2023]{Meanti2023}
Meanti, G., Chatalic, A., Kostic, V., Novelli, P., Pontil, M., and Rosasco, L.
  (2023).
\newblock Estimating koopman operators with sketching to provably learn large
  scale dynamical systems.
\newblock {\em Advances in Neural Information Processing Systems}, 36.

\bibitem[Minsker, 2017]{minsker2017}
Minsker, S. (2017).
\newblock On some extensions of {B}ernstein’s inequality for self-adjoint
  operators.
\newblock {\em Statistics \& Probability Letters}, 127:111--119.

\bibitem[Morton et~al., 2018]{Morton2018}
Morton, J., Jameson, A., Kochenderfer, M.~J., and Witherden, F. (2018).
\newblock Deep dynamical modeling and control of unsteady fluid flows.
\newblock {\em Advances in Neural Information Processing Systems}, 31.

\bibitem[Novelli et~al., 2022]{Novelli2022}
Novelli, P., Bonati, L., Pontil, M., and Parrinello, M. (2022).
\newblock Characterizing metastable states with the help of machine learning.
\newblock {\em Journal of Chemical Theory and Computation}, 18(9):5195--5202.

\bibitem[N{\"u}ske et~al., 2023]{nuske2023finite}
N{\"u}ske, F., Peitz, S., Philipp, F., Schaller, M., and Worthmann, K. (2023).
\newblock Finite-data error bounds for koopman-based prediction and control.
\newblock {\em Journal of Nonlinear Science}, 33(1):14.

\bibitem[Ostruszka et~al., 2000]{Ostruszka2000}
Ostruszka, A., Pako{\'{n}}ski, P., S{\l}omczy{\'{n}}ski, W., and
  {\.{Z}}yczkowski, K. (2000).
\newblock Dynamical entropy for systems with stochastic perturbation.
\newblock {\em Physical Review E}, 62(2):2018--2029.

\bibitem[Otto and Rowley, 2019]{Otto2019}
Otto, S.~E. and Rowley, C.~W. (2019).
\newblock Linearly recurrent autoencoder networks for learning dynamics.
\newblock {\em {SIAM} Journal on Applied Dynamical Systems}, 18(1):558--593.

\bibitem[Pascucci, 2011]{Pascucci2011}
Pascucci, A. (2011).
\newblock {\em {PDE} and {M}artingale {M}ethods in {O}ption {P}ricing}.
\newblock Springer Milan.

\bibitem[Proctor et~al., 2016]{Proctor2016}
Proctor, J.~L., Brunton, S.~L., and Kutz, J.~N. (2016).
\newblock Dynamic mode decomposition with control.
\newblock {\em {SIAM} Journal on Applied Dynamical Systems}, 15(1):142--161.

\bibitem[Raissi et~al., 2020]{Raissi2020}
Raissi, M., Yazdani, A., and Karniadakis, G.~E. (2020).
\newblock Hidden fluid mechanics: Learning velocity and pressure fields from
  flow visualizations.
\newblock {\em Science}, 367(6481):1026--1030.

\bibitem[Ross, 1995]{ross1995stochastic}
Ross, S.~M. (1995).
\newblock {\em Stochastic Processes}.
\newblock John Wiley \& Sons.

\bibitem[Schmid, 2010]{Schmid2010}
Schmid, P.~J. (2010).
\newblock Dynamic mode decomposition of numerical and experimental data.
\newblock {\em Journal of Fluid Mechanics}, 656:5--28.

\bibitem[Sch{\"u}tt et~al., 2023]{schutt2023schnetpack}
Sch{\"u}tt, K.~T., Hessmann, S. S.~P., Gebauer, N. W.~A., Lederer, J., and
  Gastegger, M. (2023).
\newblock {SchNetPack 2.0: A neural network toolbox for atomistic machine
  learning}.
\newblock {\em The Journal of Chemical Physics}, 158(14):144801.

\bibitem[Sch{\"u}tt et~al., 2019]{schutt2019schnetpack}
Sch{\"u}tt, K.~T., Kessel, P., Gastegger, M., Nicoli, K.~A., Tkatchenko, A.,
  and Müller, K.-R. (2019).
\newblock {SchNetPack: A Deep Learning Toolbox For Atomistic Systems}.
\newblock {\em Journal of Chemical Theory and Computation}, 15(1):448--455.

\bibitem[Sch\"{u}tte et~al., 2001]{Schutte2001}
Sch\"{u}tte, C., Huisinga, W., and Deuflhard, P. (2001).
\newblock {T}ransfer {O}perator {A}pproach to {C}onformational {D}ynamics in
  {B}iomolecular {S}ystems.
\newblock In {\em Ergodic Theory, Analysis, and Efficient Simulation of
  Dynamical Systems}, pages 191--223. Springer Berlin Heidelberg.

\bibitem[Schwantes and Pande, 2015]{Schwantes2015}
Schwantes, C.~R. and Pande, V.~S. (2015).
\newblock {M}odeling {M}olecular {K}inetics with {tICA} and the {K}ernel
  {T}rick.
\newblock {\em Journal of Chemical Theory and Computation}, 11(2):600--608.

\bibitem[Steinwart and Christmann, 2008]{Steinwart2008}
Steinwart, I. and Christmann, A. (2008).
\newblock {\em Support Vector Machines}.
\newblock Springer New York.

\bibitem[Tian and Wu, 2021]{Tian2021}
Tian, W. and Wu, H. (2021).
\newblock Kernel embedding based variational approach for low-dimensional
  approximation of dynamical systems.
\newblock {\em Computational Methods in Applied Mathematics}, 21(3):635--659.

\bibitem[Trefethen and Embree, 2020]{TrefethenEmbree2020}
Trefethen, L.~N. and Embree, M. (2020).
\newblock {\em {S}pectra and {P}seudospectra: {T}he {B}ehavior of {N}onnormal
  {M}atrices and {O}perators}.
\newblock Princeton University Press.

\bibitem[Tropp, 2012]{tropp2012user}
Tropp, J.~A. (2012).
\newblock User-friendly tail bounds for sums of random matrices.
\newblock Technical report.

\bibitem[Tuckerman, 2010]{tuckerman2010statistical}
Tuckerman, M. (2010).
\newblock {\em Statistical Mechanics: Theory and Molecular Simulation}.
\newblock Oxford university press.

\bibitem[Ullah and Arora, 2023]{ullah2023}
Ullah, E. and Arora, R. (2023).
\newblock Generalization bounds for kernel canonical correlation analysis.
\newblock {\em Transactions on Machine Learning Research}.

\bibitem[Williams et~al., 2015a]{Williams2015a}
Williams, M.~O., Kevrekidis, I.~G., and Rowley, C.~W. (2015a).
\newblock A data{\textendash}driven approximation of the koopman operator:
  Extending dynamic mode decomposition.
\newblock {\em Journal of Nonlinear Science}, 25(6):1307--1346.

\bibitem[Williams et~al., 2015b]{Williams2015b}
Williams, M.~O., Rowley, C.~W., and Kevrekidis, I.~G. (2015b).
\newblock A kernel-based method for data-driven {K}oopman spectral analysis.
\newblock {\em Journal of Computational Dynamics}, 2(2):247--265.

\bibitem[Wu and No{\'{e}}, 2019]{Wu2019}
Wu, H. and No{\'{e}}, F. (2019).
\newblock {V}ariational {A}pproach for {L}earning {M}arkov {P}rocesses from
  {T}ime {S}eries {D}ata.
\newblock {\em Journal of Nonlinear Science}, 30(1):23--66.

\bibitem[Yeung et~al., 2019]{Yeung2019}
Yeung, E., Kundu, S., and Hodas, N. (2019).
\newblock Learning deep neural network representations for koopman operators of
  nonlinear dynamical systems.
\newblock In {\em 2019 American Control Conference ({ACC})}. {IEEE}.

\end{thebibliography}
}

\newpage
    
\appendix

\begin{center}
{\Large \bf Appendix} 
\end{center}

The appendix is organized as follows.

\begin{itemize}
    \item  App. \ref{app:backgound} reviews the notation used throughout the paper and basic notions of stochastic processes and reproducing kernel Hilbert spaces.
    \item In App. \ref{app:prob} we present the theoretical background for material presented in Sec.~\ref{sec:prob}.
    \item In App. \ref{app:score} we prove the first main result of the paper, Thm.~\ref{thm:main}. 
    \item App. \ref{app:methods} adds value to Sec.~\ref{sec:methods}. We first give theoretical justification for using the empirical score instead of the population one, and then provide details on how to address downstream tasks. 
    \item In App. \ref{app:SDE} we prove the second main result of the paper, Thm.~\ref{thm:main_cont} about DPNets for continuous time stochastic processes. 
    \item Finally, App. \ref{app:exp} contains more information about the experiments and additional results.
\end{itemize}

\section{Background and notation}\label{app:backgound}

\subsection{Stochastic processes tools}
Let $\spX$ be a state space and $\sigalg$ be the Borel $\sigma$-algebra on $\spX$. Let $\mathcal{M}$ denote the space of all finite measures on $(\spX,\sigalg)$. In what follows, we consider a random process $(X_{t})_{t\in\mathbb{T}}$
with values in a state space $\spX$.

\begin{definition}[Markov process.]
Let $\mathcal{F}_{t} := \Sigma\left(\{X_{s} : 0 \leq s \leq t\}\right)$ be the sigma-algebra generated by the elements of the process up to time $t$. The random process $\mathbf{X}=(X_{t})_{t\in\mathbb{T}}$ is called a Markov process if for all $t, \tau \geq 0$ and $B \in \sigalg$ it holds
\begin{equation}
    \PP \left(X_{t + \tau} \in B \middle|  \mathcal{F}_{t}\right) = \PP \left(X_{t + \tau} \in B \middle| X_{t}\right),
\end{equation}
that is, conditioning on the full history of the process $\mathcal{F}_{t}$ up to time $t$ is equivalent to condition only on the state of the process at time $t$. For this reason, Markov process are sometimes called ``memoryless''.
\end{definition}

\begin{definition}
    The transition density function $p_{\tau}\,:\, \spX\times \spX \rightarrow [0,\infty]$ of a time-homogeneous process $\mathbf{X}$ is defined by
    $$
    \PP\left(  X_{t+\tau} \in B \vert X_t = x\right) = \int_{B} p_{\tau}(x,y)dy,
    $$
    for every measurable set $B\in \sigalg$.
\end{definition}
The distribution of a time-homogeneous stochastic process $(X_t)_t$ with transition density functions $(p_{\tau})_{\tau>0}$ can be described by the semigroup of transfer operators $(\TO_{\tau})_{\tau\geq 0}$ usually defined on  $L^{\infty}(\spX)$.

\begin{definition}[Transfer operator]
For any $\tau\geq 0$, the Koopman transfer operator $\TO_\tau\, :\, L^{\infty}(\spX)\rightarrow L^{\infty}(\spX)$ is defined by
$$
\TO_\tau(f) = \EE\left[ f(X_{t+\tau})\vert X_t = \cdot\right], \quad \forall f\in L^{\infty}(\spX).
$$
\end{definition}

\begin{definition}[Feller process]
    Let $(\TO_{\tau})_{\tau\geq 0}$ be the semigroup of transfer operators of a homogeneous Markov process. Then, $(\TO_{\tau})_{\tau\geq 0}$ is said to be a \emph{Feller} semigroup when the following condition holds:
    \begin{enumerate}
        \item $\TO_{\tau} \left(\mathcal{C}_0(\spX)\right) \subseteq \mathcal{C}_0(\spX)$ for all $\tau \geq 0$.
        \item $\lim_{\tau \to 0} \| \TO_{\tau}f - f \| = 0$ for all $f \in \mathcal{C}_0(\spX)$.
    \end{enumerate}
    Here $\mathcal{C}_0(\spX)$ is the Banach space of all continuous functions vanishing at infinity.
\end{definition}

\begin{definition}[Infinitesimal generator of a semigroup]
    Let $(\TO_{\tau})_{\tau\geq 0}$ be a Feller semigroup. We define its infinitesimal generator $\LO : \mathcal{C}_0(\spX) \to \mathcal{C}_0(\spX)$ 
    $$
    \LO f = \lim_{\tau \rightarrow 0}\frac{1}{\tau}(\TO_{\tau}f - f),
    $$
    for any $f\in \mathcal{C}_0(\spX)$ such that the above limit is well-defined. 
\end{definition}
The above definitions can be lifted to $L^{2}_{\im}(\spX)$ if the dynamical system has an invariant distribution, or if we are interested in the action of $\TO_{\tau}$ on a specific couple of states $X, X^\prime$ of the process, as discussed in the main text.

\subsection{Reproducing kernel Hilbert spaces}\label{app:kernels}

In this section we review few basic concepts on kernel based approaches to learning transfer operators; for more details on 
reproducing kernel Hilbert spaces (RKHS) we refer reader to~\citep{aron1950}.

Let $\kHX \colon \spX\times\spX\to\R$ and $\kHY \colon \spY\times\spY\to\R$ be two bounded kernels,  and let $\spHX$ and $\spHY$ be their respective reproducing kernel Hilbert spaces (RKHS). Denote the canonical feature maps by $\fHX(x):=\kHX(x,\cdot)$, $x\in\spX$, and $\fHY(x'):=\kHY(x',\cdot)$, $x'\in\spY$.

Next, consider two probability measures $\mX$ and $\mY$ on $\spX$, and their associated $L^2$ spaces $\LiiX$ and $\LiiY$. If $\kHX(\cdot,\cdot)$ is square-integrable w.r.t. measure $\mX$ and $\kHY(\cdot,\cdot)$ is square-integrable w.r.t. the measure $\mY$, then we can define injection (or the evaluation) operators $\TSX\colon\spHX \hookrightarrow \LiiX$ and $\TSY\colon\spHY \hookrightarrow \LiiY$. Then, their adjoints   $\TSX^*\colon \LiiX\to  \spHX$ and $\TSY^*\colon \LiiY \to \spHY$ are given by
\[
\TSX^*f = \int_{\spX} f(x)\fHX(x)\mX(dx)\in\spHX\quad\text{ and }\quad \TSY^*f' = \int_{\spY} f'(x')\fHY(x')\mY(dx')\in\spHY,
\]
where $f\in\LiiX$ and $f'\in\LiiY$. 

Using injections and their adjoints one can introduce covariance operators $\Cx\colon \spHX\to\spHX$ and $\Cy\colon \spHY\to\spHY$ by
\[
\Cx :=  \TSX^*\TSX=\EE_{\rvX\sim\mX} \fHX(\rvX)\otimes \fHX(\rvX), \quad \text{ and }\Cy :=  \TSY^*\TSY=\EE_{\rvY\sim\mY} \fHX(\rvY)\otimes \fHX(\rvY), 
\] 
as well as the cross-covariance $\Cxy\colon \spHY\to\spHX$ operator 
\[
\Cxy := \TSX^*\TO \TSY = \EE_{(\rvX,\rvY)\sim\mXY} \fHX(\rvX)\otimes \fHY(\rvY),
\]
where $\mXY$ is the joint measure of $(\rvX,\rvY)$ and $\TO:\LiiY\to\LiiX$ is a transfer operator corresponding to the evolution of $\rvX\sim\mX$ to $\rvY\sim\mY$. 

Different estimators for the problem of learning transfer operators from data $\Data = (x_i,x_i')_{i\in[n]}$ have been considered, see e.g.~\citep{Kostic2022,Li2022}, which are all of the form $\EEstim_W = \ESY^* W \ESX$,  where $W\in\R^{n\times n}$ and $\ESX\colon \spHX\to\R^{n}$ and $\ESY\colon \spHY\to\R^{n}$ are sampling operators given by $\ESX f = n^{-\frac{1}{2}} [ f(x_1)\,\ldots\, f(x_n)]^\top$, $f\in\spHX$, and $\ESY f' = n^{-\frac{1}{2}} [ f'(x_1')\,\ldots\, f'(x_n')]^\top$, $f'\in\spHY$. In particular, kernel methods usually consider \emph{universal kernels} that generate infinite-dimensional RKHS spaces that are dense in $L^2$. In such a way, one can approximate transfer operators via operator (vector-valued) regression arbitrarily well with enough data. However, one important aspect has been recently reported in \citep{Kostic2022, kostic2023koopman}. Namely, from the application perspective, utility of transfer operators is largely relying on the ability to estimate well their eigenvalues and eigenfunctions, which lead to the notion of modal decomposition of observables \citep{Brunton2022}. In that context, the difference in the norms between RKHS and $L^2$ spaces can produce spectral pollution and lead to bad estimation of the eigenvalues and eigenfunctions. We recall that the difference in norms is due to the covariance, i.e. for every $h\in\spH$ its L2 norm is 
\begin{equation}\label{eq:norm_change}
\int_{\spX} |h(x)|^2\mX(dx) = \scalarp{\TSX h, \TSX h}_{\LiiX} = \scalarp{h,\Cx h}_{\spHX}.    
\end{equation}
While this difference in norms is unavoidable for universal kernels (indeed for universal kernels generating infinite-dimensional RKHS spaces, since covariance is a trace class operator, we have $\eval_j(\Cx)\rightarrow 0 $ as $j\rightarrow \infty$ potentially leading to $\scalarp{h,\Cx h}_{\spHX}\ll \scalarp{h,h}_{\spHX}$ for some $h\in \spHX$), for finite dimensional kernels one can hope to avoid metric distortion between spaces by aiming to find features for which $\Cx=I$.

\section{Representation Learning: problem and approach}\label{app:prob}

{In this section we prove some of the statements made in Sec.~\ref{sec:prob}. In particular we show that \eqref{eq:bias_bound} holds true, and show how the objectives \eqref{eq:optimization} and \eqref{eq:optimization_continious} are related to the desired space $\spH$.} 

{We first note that $\TO_{\vert_{\spH}}:=\TO \TS \colon\spH\to\Lii$, while 
$\TO \proj{\spH} = \TO \TS \Cx^\dagger \TS^*$. Hence, 
$\norm{[I-\proj{\spH}]\TO\proj{\spH}}^2 = \norm{[I-\proj{\spH}]\TO_{\vert_{\spH}} \Cx^{\dagger/2}}^2 = \norm{\Cx^{\dagger/2} ([I-\proj{\spH}]\TO_{\vert_{\spH}})^*}^2$. Therefore, using that $\range\big(([I-\proj{\spH}]\TO_{\vert_{\spH}})^*\big)= \range(\TS^* \TO^* [I-\proj{\spH}]) \subseteq\range(\Cx^{\dagger})$ \eqref{eq:bias_bound} follows. }

Next, we show how Eckart-Young's theorem justifies \eqref{eq:optimization}, which is the basis of Thm.~\ref{thm:main}.

\begin{lemma}\label{lm:projections}
Let $\spH_1$ and $\spH_2$ be two separable Hilbert spaces, and $A\colon \spH_1\to\spH_2$ be a compact operator. Let $r\in\N$ and $\mathcal{P}^r_k:=\{ P \colon \spH_k\to\spH_k\,\vert\, P^*= P^2 = P, \rank(P)\leq r \}$ denote the set of rank-$r$ orthogonal projectors in $\spH_k$, $k=1,2$. It holds that
\begin{enumerate}[label=(\roman*)]
\item\label{lm:projections_i} $\displaystyle{\max_{(P_1,P_2)\in\mathcal{P}_1^r \times \mathcal{P}_2^r }\hnorm{P_2AP_1}^2 } = \sum_{i\in[r]}\sigma_i^2(A)$,
\item\label{lm:projections_ii} Let $U_r \Sigma_r V_r^*$ be the $r$-truncated SVD of $A$, then 
$(V_r V_r^*,U_r U_r^*)\in\displaystyle{\argmax_{(P_1,P_2)\in\mathcal{P}_1^r \times \mathcal{P}_2^r }\hnorm{P_2AP_1}^2 }$. 
\item\label{lm:projections_iii} If $\sigma_{r+1}(A)<\sigma_r(A)$ and $\hnorm{A}<\infty$, then $\displaystyle{\argmax_{(P_1,P_2)\in\mathcal{P}_1^r \times \mathcal{P}_2^r }\hnorm{P_2AP_1}^2 }$ is singleton.
\end{enumerate}
\end{lemma}
\begin{proof}
This lemma is a direct consequence of the Eckart-Young theorem for compact operators~\citep{Kato76}. Recall, since $A$ is compact, there exists its SVD $A=U\Sigma V^*$ and Eckart-Young theorem states that if $A$ is Hilbert-Schmidt, for every $r\in\N$ and every $B\colon \spH_1\to\spH_2$ such that $\rank(B)\leq r$, it holds $\hnorm{A-B}\geq \hnorm{A-A_r}$, where $A_r:= U_r \Sigma_r V_r^*$ denotes the $r$-truncated SVD of A. Moreover,  if $\sigma_{r+1}(A)<\sigma_r(A)$, then equality implies $B=A_r$.

Hence, for arbitrary $P_k\in\mathcal{P}_k^r$, $k=1,2$, using that the norm of a projection is 1 and Pythagoras theorem, for every $m\geq r$ we have that
\begin{align*}
\hnorm{P_2 A_m P_1 }^2 & \leq \hnorm{P_2 A_m}^2 \norm{P_1}^2 = \hnorm{P_2 A_m}^2 =  \hnorm{A_m}^2 - \hnorm{[\Id - P_2]A_m}^2 \\
& = \hnorm{A_m}^2 - \hnorm{A_m-P_2 A_m}^2 \leq \hnorm{A_m}^2 - \hnorm{A_m-A_r}^2  = \hnorm{A_r}^2, 
\end{align*}
and, similarly, $\hnorm{P_2 A_m P_1 }^2 \leq \hnorm{A_m}^2 - \hnorm{A_m- A_m P_1}^2 \leq \hnorm{A_r}^2$.
However, since
\begin{align*}
\big\lvert \hnorm{P_2 A P_1} - \hnorm{P_2 A_m P_1} \big\rvert & \leq \hnorm{P_2 (A-A_m) P_1} \leq \sqrt{r} \norm{P_2 (A-A_m) P_1} \\
& \leq \sqrt{r} \norm{A-A_m} \leq \sqrt{r} \sigma_{m+1}(A),
\end{align*} 
we obtain $\hnorm{P_2 A P_1 }^2 \leq \hnorm{A_r}^2 + \sqrt{r} \sigma_{m+1}(A)$ for all $m\geq r$. Thus, since $\hnorm{P_2 A P_1 }^2 = \hnorm{A_r}^2$ obviously holds for  $P_1 = V_r V_r^*$ and $P_2 = U_r U_r^*$, letting $m\to\infty$ we obtain \ref{lm:projections_i} and \ref{lm:projections_ii}. 

Finally, assume that $A$ is an Hilbert-Schmidt operator such that $0\leq\sigma_{r+1}(A) <\sigma_r(A)$. Then, similarly to the above now working with $A$ instead of $A_m$, 
\[
\hnorm{P_2 A P_1 }^2 \leq \hnorm{A}^2 - \max\{\hnorm{A- A P_1}^2, \hnorm{A- P_2 A}^2\} \leq \hnorm{A_r}^2.
\]
So, if 
\[
(P_1,P_2)\in\displaystyle{\argmax_{(P_1',P_2')\in\mathcal{P}_1^r \times \mathcal{P}_2^r }\hnorm{P_2'AP_1'}^2 },
\]
then 
\[
\max\{\hnorm{A- A P_1}^2, \hnorm{A- P_2 A}^2\} = \hnorm{A-A_r}^2.
\]
Now, the uniqueness result of the Eckart-Young theorem implies that $A P_1 = P_2 A = A_r$, and, consequently $\rank(P_1) =\rank(P_2) = r$, $A_r^\dagger A P_1 = A_r^\dagger A_r$ and $P_2 A A_r^\dagger  = A_r A_r^\dagger$. Therefore, $V_r V_r^* P_1 = V_r V_r^*$ and $P_2 U_r U_r^* = U_r U_r^*$, imply $\range(V_r)\subseteq \range(P_1)$ and  $\range(U_r){\subseteq} \range(P_2)$. But since $P_1$ and $P_2$ have exactly rank $r$, we obtain $P_1 = V_r V_r^*$ and $P_2 = U_r U_r^*$.
\end{proof}

We also discuss how to extend part of the analysis to the important setting of stationary deterministic systems for which the transfer operator is not compact.
\begin{remark}\label{rem:deterministic} 
The compactness assumption on the the transfer operator, does not hold for purely deterministic dynamical systems. However, our 
approach 
is still applicable to the study of deterministic systems, that is when $X_{t+1} = F(X_t)$,
for a deterministic map $F\colon \spX\to \spX$, and $X_t\sim\im$ for all $t\in\mathbb{T}$, where $\im$ is the invariant measure defined on the attractor. Then, we have that $\EE[ f(X_{t+1})\,\vert\, X_t] = F\circ f$ and $\TO\colon\Lii\to\Lii$ is unitary, see e.g.~\citep{Brunton2022}. Thus, $\TO$ is not a compact operator, however it is a normal one.
But then, Pythagoras theorem gives that
\[
\inf_{\param\in\spParam}\hnorm{[\Id-\proj{\spHXp}]\TO\proj{\spHXp}]}^2 = \inf_{\param\in\spParam} \left(\hnorm{\TO\proj{\spHXp}}^2 - \hnorm{\proj{\spHXp}\TO\proj{\spHXp}}^2\right) = r - \sup_{\param\in\spParam} \score(\param), 
\]  
where the second equality holds since the HS-norm is unitarily invariant. Therefore, when $w$ is  such that $\TO\embXp{,j} = \eval_j \embXp{,j}$, since $\TO$ is unitary we have $\score(\param) = \sum_{j=1}^{r} \abs{\lambda_j}^2 = r$
and the above inf is zero. Consequently,  we have identified an $r$-dimensional invariant subspace of $\TO$.
\end{remark}

Finally, we conclude this section with the lemma that is the basis of generator learning via optimization problem \eqref{eq:optimization_continious} whose objective is defined through a partial trace. 

\begin{lemma}\label{lm:partial_trace}
Let $\spH$ be a separable Hilbert space and let $\spH_r\subseteq \spH$ be a (finite) $r$-dimensional subspace. If $A\colon \spH\to\spH$ is a self-adjoint operator having at least $r+1$ eigenvalues $\eval_1(A)\geq\eval_2(A)\geq\ldots\geq \eval_r(A)\geq\eval_{r+1}(A) $ above its essential spectrum, then
\[
\tr(A \proj{\spH_r}) \leq \eval_1(A)+\eval_2(A)+\cdots+\eval_r(A),
\]
and equality holds when $\spH_r$ is spanned by eigenfunctions of $A$ corresponding to eigenvalues $\eval_1(A),\ldots, \eval_r(A)$.
\end{lemma}
\begin{proof}
First, let $(v_i)_{i\in[r]}$ denote eigenvectors of $A$ corresponding to eigenvalues $\eval_i$, that is $Av_i = \eval_i(A) v_i$, and let $(h_i)_{i\in\N}$ be an ortho-normal basis of $\spH$ such that $(h_i)_{i\in[r]}$ is the ortho-normal basis of $\spH_r$. Then, clearly
\begin{equation}\label{eq:partial_trace}
    \tr(A\proj{\spH_r})= \tr(\proj{\spH_r}A\proj{\spH_r}) = \sum_{i\in\N} \scalarp{h_i,(\proj{\spH_r}A\proj{\spH_r})h_i} = \sum_{i\in[r]} \scalarp{h_i,A h_i},
\end{equation}
which is clearly equal to $\eval_1(A)+\eval_2(A)+\ldots+\eval_r(A)$ whenever $h_i=v_i$, $i\in[r]$.

The upper bound we prove by induction. First, for $r=1$ we have that $\scalarp{h_i,Ah_i}\leq \eval_1(A)$ follows directly from the Courant-Fischer max-min theorem for operators which claims that
\[
\eval_i(A) = \max_{h_1,\ldots,h_i} \min\{\scalarp{h,A h}\,\vert\, h\in\Span(h_j)_{j\in[i]},\norm{h}=1\},\;i\in[r+1]
\]
Next, assuming that \eqref{eq:partial_trace} holds for arbitrary $A$ and $r\leq m-1$, we will prove that it holds for $r=m$.

Start by observing that there exists $\spH'\subseteq\spH_{m}$ such that $\dim(\spH')\leq m-1$ and $\spH'\perp v_1$. Indeed, taking $g_i =\sum_{j\in[m]} b_{ij} h_j $, $i\in[m-1]$, for some $B=[b_{ij}]\in\R^{(m-1) \times m}$, we have that $g_i\perp v_1$ for all $i\in[m-1]$ iff $B \beta = 0$, where the vector $\beta\in\R^{m}$ is given by $\beta_j= \scalarp{h_j,v_1}$. Now, if $\beta=0$, then $\spH_m \perp v_1$ and the claim trivially follows. Otherwise, since $\dim(\Span(\beta)^\perp)=m-1$, there exists a matrix $B\in\R^{(m-1)\times m}$ so that $B\beta =0$. Consequently, the space $\spH'$ spanned by $(g_i)_{i\in[m-1]}$ is such that $\dim(\spH')\leq m-1$ and $\spH'\subseteq\Span(v_1)^\perp$. Without the loss of generality, assume that $(g_i)_{i\in[m-1]}$ are orthonormal basis of $\spH'$ and that $g_{m}$ is such that $(g_i)_{i\in[m]}$ are orthonormal basis of $\spH$.

First, clearly $\scalarp{g_{m},A g_{m}} \leq \eval_1$. On the other hand,  define $A':=A+(\eval_{r+1} - \eval_1)\,v_1\otimes v_1$. This is an operator obtained by deflating (moving) the first eigenvalue into the $(r+1)$-th one. Namely, we have that $\eval_i(A')=\eval_{i+1}(A)$ for $i\in[m-1]$. Moreover, it holds that $A'\proj{\spH'} = A \proj{\spH'}$. Hence, according to the inductive hypothesis, we have
\[
\tr(A\proj{\spH'}) = \tr(A'\proj{\spH'}) \leq \sum_{i\in[m-1]} \eval_i(A') = \sum_{i=2}^{m} \eval_i(A),   
\]
and, consequently, 
\[
\tr(A\proj{\spH}) =  \sum_{i\in[r]} \scalarp{g_i,A g_i} = \tr(A\proj{\spH'}) + \scalarp{g_m,A g_m} \leq  \sum_{i=2}^{m} \eval_i(A) + \eval_1(A) = \sum_{i\in[m]} \eval_i(A),   
\]
which completes the proof.
\end{proof}

\section{Score functional}\label{app:score}

We start by extending problem \eqref{eq:optimization} to the representation learning for stable non-stationary processes $(\rvX_t)_{t\in \N_0}$. Since for the operator regression problem we typically collect data samples uniformly at random along trajectory, we have that $\rvX \sim \mX := \frac{1}{n} \sum_{i\in[n]}\mXt{i-1}$, recalling that $\mXt{t}$ is the law of $\rvX_t$ and $n$ is the sample size. Then, the one step ahead evolution of a random variable $\rvX$ is $\rvY\sim\mY=\frac{1}{n} \sum_{i\in[n]}\mXt{i}$. Since the transfer operator is properly defined as $\TO\colon\LiiY\to\LiiX$, we obtain the following version of \eqref{eq:bias_bound}
\begin{equation}\label{eq:bias_bound_nonstat}
\norm{[\Id-\proj{\spHX}]\TO \proj{\spHX}'}^2 \lambda_{\min}^+(C_{\spHX}') \leq \norm{[\Id-\proj{\spHX}]\TO_{\vert_{\spHX}}}^2 \leq \norm{[\Id-\proj{\spHX}]\TO \proj{\spHX}'}^2 \lambda_{\max}(C_{\spHX}'), 
\end{equation}
where, $\proj{\spHX}\colon \LiiX\to\LiiX$ and $\proj{\spHX}'\colon \LiiY\to\LiiY$ are the orthogonal projections onto $\spH$ as a subspace of $\LiiX$ and $\LiiY$, respectively, and covariance $C_{\spHX}'$ is now w.r.t. measure $\mY$. Now, assuming that the feature map of $\spH$ is  bounded by some constant $c_{\spH}$, we have that 
\[
\norm{C_{\spHX}' - C_{\spHX}}\leq \int_{\spX} \norm{\fHX(x)\otimes \fHX(x)} \abs{\mY-\mX}(dx) \leq 2 c_{\spH}^2 \norm{\mY-\mX}_{TV} = \frac{2 c_{\spH}^2}{n} \norm{\mXt{n}-\mXt{0}}_{TV}.
\]
Thus, if the dynamics is stable, the total variation norm $\norm{\mXt{n}-\mXt{0}}_{TV}$ is bounded w.r.t. $n\to\infty$, and we conclude that for large enough sample size $n$ one can replace $C_{\spH}'$ by $C_{\spH}$ in \eqref{eq:bias_bound_nonstat} to obtain tight approximation of the approximation error $\norm{[\Id-\proj{\spHX}]\TO_{\vert_{\spHX}}}^2$ by controlling $\norm{[\Id-\proj{\spHX}]\TO \proj{\spHX}'}$. This motivates  the optimization problem for non-stationary dynamics
\begin{equation}\label{eq:optimization_nonstat}
\max_{\spHX\subset\LiiX,\spHY\subset\LiiY}\left\{\hnorm{\proj{\spHX}\TO\proj{\spHY}'}^2\;\vert\; C_{\spHX} = C_{\spHY}' = I, \dim(\spHX)\leq r,\,\dim(\spHY)\leq r\right\}.
\end{equation}

Next, we prove \eqref{eq:score}, relating the projection score with the covariance operators.

\begin{lemma}\label{lm:proj_to_cov}
Let $\rvX$ and $\rvY$ be two $\spX$-valued random variables distributed according to probability measures $\mX$ and $\mY$, respectively. Given $\param\in\spParam$ let for every $j\in[r]$ functions $\embXp{,j}\colon\spX\to\R$ and $\embYp{,j}\colon\spX\to\R$ be square integrable w.r.t measures $\mX$ and $\mY$, respectively. If $\proj{\spHXp}\colon \LiiX\to\LiiX$ and $\proj{\spHYp}'\colon\LiiY\to\LiiY$ are orthogonal projections onto subspaces $\spHXp:=\Span(\embXp{,j})_{j\in [r]}$ and $\spHYp:=\Span(\embYp{,j})_{j\in [r]}$, respectively, then
\begin{equation}\label{eq:proj_to_cov}
    \hnorm{\proj{\spHXp}\TO\proj{\spHYp}'}^2=
\hnorm{(\Cxp)^{\dagger/2}\Cxyp (\Cyp)^{\dagger/2}}^2.
\end{equation}
\end{lemma}
\begin{proof}
The proof directly follows from the notion of finite-dimensional RKHS. 
Let $\kHXp \colon \spX\times\spX\to\R$ and $\kHYp \colon \spY\times\spY\to\R$ be two kernels given by 
\[
\kHXp(x,y):= \embXp{}(x)^\top \embXp{}(y)\quad\text{and}\quad \kHYp(x,y):= \embYp{}(x)^\top \embYp{}(y), \;x,y\in\spX.
\]
Then $\spHXp$ and $\spHYp$ are the reproducing kernel Hilbert spaces (RKHS) associated with kernels $\kHXp$ and $\kHYp$, respectively. 

Now, due to square integrability of the embeddings $\embXp{,j}$ and $\embYp{,j}$, $j\in[r]$, we have that the injection operators of the two RKHS spaces into their respective $L^2$ spaces are well-defined: $\TSX\colon\spHXp \hookrightarrow \LiiX$ and $\TSY\colon\spHYp \hookrightarrow \LiiY$.
Moreover, observing that $\spHXp$ and $\spHYp$ are isometrically isomorphic to $\R^r$, we have that $\TSX^*\TSX \colon\spHXp\to\spHXp$ and $\TSY^*\TSY \colon\spHYp\to\spHYp$ can be identified with $\Cxp\in\R^{r\times r}$ and $\Cyp\in\R^{r\times r}$, respectively, that is
\[
\TSX^*\TSX = Q \Cxp Q^*\quad\text{ and }\quad \TSY^*\TSY = Q' \Cyp (Q')^*.
\]
where $Q\colon\R^{r}\to\spHXp$ and $Q'\colon\R^{r}\to\spHYp$ are partial isometries, meaning that $QQ^*$ and $Q'(Q')^*$ are identity operators on $\spHXp$ and $Q^*Q$ and $(Q')^*Q'$ are identity matrices in $\R^{r\times r}$.

As a consequence, the polar decompositions of finite rank operators $\TSX$ and $\TSY$ can be written as
\begin{equation}\label{eq:polar}
\TSX = U Q (\Cxp)^{1/2}Q^*\quad\text{ and }\quad \TSY = U'Q' (\Cyp)^{1/2}(Q')^*,    
\end{equation}
where $U\colon\spHXp\to\LiiX$ and $U'\colon\spHYp\to\LiiY$ are partial isometries. 

But then, since $\spHXp$ as a subspace of $\LiiX$ is identified as $\range(\TSX)$, and $\spHYp$ as a subspace of $\LiiY$ is identified as $\range(\TSY)$, using adjoints  $\TSX^*\colon \LiiX\to  \spHXp$ and $\TSY^*\colon \LiiY \to \spHYp$, we can write the aforementioned orthogonal projections as
\begin{equation}\label{eq:projisomet}
\proj{\spHXp} = \TSX (\TSX^*\TSX)^\dagger \TSX^*  \quad\text{ and }\quad\proj{\spHYp} = \TSY (\TSY^*\TSY)^\dagger \TSY^*,
\end{equation}
which, using \eqref{eq:polar} and the fact that $(\Cxp)^{1/2}(\Cxp)^\dagger = (\Cxp)^{\dagger/2}$, is equivalent to
\[
\proj{\spHXp} =  U Q (\Cxp)^{\dagger/2}Q^*\TSX^*  \quad\text{ and }\quad\proj{\spHYp} = U' Q' (\Cyp)^{\dagger/2}(Q')^*\TSY^*.
\]

Finally, since for every $v\in\R^r$ we obtain
\begin{align*}
(Q^*\TSX^* \TO \TSY Q') v & = Q^* \TSX^* (\TO \TSY Q'v) = Q^* \int_{\spX}\mX(dx)  \fHXp(x) (\TO \TSY Q'v)(x) \\
&=Q^* \int_{\spX}\int_{\spX} \mX(dx)p(x,dx') \fHXp(x) (Q'v)(x')\\
&=Q^*\int_{\spX\times\spX} \jim(dx,dx') \fHXp(x) (Q'v)(x')
\end{align*}
where $\jim$ is joint measure of $(\rvX,\rvY)$ and $\fHXp$ is the canonical feature map of $\kHXp$, that is $\fHXp(x)=\kHXp(\cdot,x)$, $x\in\spX$. Next, using the reproducing property 
\begin{align*}
(Q^*\TSX^* \TO \TSY Q') v & =Q^*\! \int_{\spX\times\spX} \jim(dx,dx') \fHXp(x) \scalarp{\fHYp(x'),Q'v}_{\spHYp} \\
& = \left[ \int_{\spX\times\spX} \jim(dx,dx')Q^*\fHXp(x) \otimes ((Q')^*\fHYp(x')) \right]v\\
& = \left[ \int_{\spX\times\spX} \jim(dx,dx')\embXp{}(x) \otimes \embYp{}(x') \right] v = \Cxyp v,
\end{align*}
where $\fHYp$ is the canonical feature map of $\kHYp$ and we used that $Q^*\fHXp(x)=\embXp{}(x)$ and $(Q')^*\fHYp(x')=\embYp{}(x')$.

Therefore, using \eqref{eq:projisomet} we obtain
\[
\proj{\spHXp}\TO\proj{\spHYp}' =
U Q (\Cxp)^{\dagger/2}\Cxyp (\Cyp)^{\dagger/2} (Q')^* (U')^*,
\]
which using that $U,U'$ and $Q,Q'$ are partial isometries, implies \eqref{eq:proj_to_cov}.
\end{proof}

Now, using Lem.~\ref{lm:projections}, we prove one of the main theoretical results of the paper. From this point forward, we abbreviate $\score(w):=\score^0(w)$, $\rscore(w):=\rscore^0(w)$, $\escore_n(w):=\score^0_n(w)$ and $\erscore_n(w):=\score^0_n(w)$.

\thmScore*
\begin{proof}
Recall that  $\score^\reg(\param)= \score(\param) - \reg \left(\Reg(\Cxp) + \Reg(\Cyp)\right)$ with $\score(\param) = \hnorm{\proj{\spHXp}\TO\proj{\spHYp}'}^2=
\hnorm{(\Cxp)^{\dagger/2}\Cxyp (\Cyp)^{\dagger/2}}^2$, where the last equality follows from \Cref{lm:proj_to_cov}, and that $\rscore(\param):= \hnorm{\Cxyp}^2/\big(\norm{\Cxp}\norm{\Cyp}\big)$. 
The first inequality in \eqref{eq:score_bound}%
holds thanks to a standard matrix norm inequality, while the second holds by applying \Cref{lm:projections}\ref{lm:projections_i} and noting that $\Reg(\Cxp) + \Reg(\Cyp) \geq 0$.

Now assume that $(\embXp{,j})_{j\in[r]}$ and $(\embYp{,j})_{j\in[r]}$ are some leading $r$ left and right singular functions of $\TO$, respectively. Then, since singular functions form ortho-normal systems in $L^2$ spaces we have that 
\[
(\Cxp)_{i,j} = \scalarp{\embXp{,i},\embXp{,j}}_{\LiiX} = \delta_{i,j}\;\text{ and }\; (\Cyp)_{i,j} = \scalarp{\embYp{,i},\embYp{,j}}_{\LiiX} = \delta_{i,j},\,i,j\in[r],
\]
that is $\Cxp=\Cyp = \Id_{r}$, and, therefore, $\Reg(\Cxp)= \Reg(\Cyp) = 0$ and $\rscore(\param) = \score(\param)$. Thus, using \Cref{lm:projections}\ref{lm:projections_ii}, we obtain that \eqref{eq:score_bound} holds with equalities in place of inequalities.

Next, assume that the operator $\TO$ is Hilbert-Schmidt, $\sval_{r}(\TO)>\sval_{r+1}(\TO)$, $\reg>0$ and $\rscore^\reg(\param) = \sval_1^2(\TO)+\ldots+\sval_r^2(\TO)$. Then, clearly $\Reg(\Cxp)= \Reg(\Cyp) = 0$, which implies that $\Cxp=\Cyp=\Id_r$, that is $(\embXp{,j})_{j\in[r]}$ and $(\embYp{,j})_{j\in[r]}$ form orthonormal systems.
Consequently, 
\[
\proj{\spHXp} = \sum_{j\in[r]} \embXp{,j} \otimes \embXp{,j}\quad \text{ and }\quad \proj{\spHYp} = \sum_{j\in[r]} \embYp{,j}\otimes \embYp{,j},
\]
and
\[
\score(\param) = \score^\reg(\param) = \rscore^\reg(\param) = \sval_1^2(\TO)+\ldots+\sval_r^2(\TO).
\]
But then, \Cref{lm:projections}\ref{lm:projections_iii} implies that $\proj{\spHXp}$ and $\proj{\spHYp}$ are orthogonal projectors on the leading $r$ left and right singular spaces of $\TO$. In other words, up to unitary changes of basis, $(\embXp{,j})_{j\in[r]}$ and $(\embYp{,j})_{j\in[r]}$ are the leading $r$ left and right singular functions of $\TO$, which completes the proof.
\end{proof}

We conclude this section with a remark on the importance of regularization. 

\begin{remark}\label{rem:regularization}
Recalling the proof of the previous theorem, note that when $\reg=0$, equality in \eqref{eq:score_bound} is achieved whenever $\spHXp$ and $\spHYp$ are \emph{spanned} by leading $r$ left and right singular functions of $\TO$, respectively. Meaning that after the change of basis $Q,Q'\in\R^{r\times r}$ so that $(\sum_{i\in[r]}Q_{i,j}\embXp{,i})_{j\in[r]}$ and $(\sum_{i\in[r]}Q'_{i,j}\embYp{,i})_{j\in[r]}$ are the leading left and right singular functions, respectively. Under the additional conditions, also the "only if" part holds for some changes of basis. Indeed, we can take the change of basis to be $Q = (\Cxp)^{-1/2}$ and  $Q' = (\Cyp)^{-1/2}$, which without regularization term, need not be unitary. This, as we see in App. \ref{app:prediction}, highly impacts on the stability of computation of the transfer operator estimators, and, therefore, impacts their practical use. 
\end{remark}

\section{Methods}\label{app:methods}

\subsection{Statistical learning guarantees}
\label{app:concScore}

Before we study the statistical learning guarantees for our novel score $\rscore$, we first discuss a fundamental limitation of using the score $\score$ to learn the invariant representation. In order to maximize the score $\mathcal{P}$ one can use standard ridge regularization on the empirical covariances. This approach considered in DeepCCA \citep{andrew2013deep} and VAMPNets \citep{Mardt2019}, typically requires a large number of training samples $n$, and the rates are governed by the choice of the regularization hyperparameter, typically ranging from $O(n^{-1/3})$ to $O(n^{-1/2})$. Namely, in this approach one uses score $\hnorm{( \ECxp+\lambda I)^{-1/2}\ECxyp( \ECyp+\lambda I)^{-1/2}}$  instead of $\score$, where $\lambda$ is the regularization parameter. So the main concern is now to measure how close the regularized empirical score is to the true score? To do this, we first study the operator norm deviation:
\begin{align*}
&\left| \norm{(\Cxp)^{\dagger/2}\Cxyp (\Cyp)^{\dagger/2}} -   \norm{( \ECxp+\lambda I)^{-1/2}\ECxyp( \ECyp+\lambda I)^{-1/2}} \right|\\
&\hspace{0.5cm}\leq \norm{ (\Cxp)^{\dagger/2}\Cxyp (\Cyp)^{\dagger/2} -   (\ECxp+\lambda I)^{-1/2}\ECxyp( \ECyp+\lambda I)^{-1/2}}\\
&\hspace{0.5cm}\leq \norm{ (\Cxp)^{\dagger/2}\Cxyp (\Cyp)^{\dagger/2} -   (\Cxp+\lambda I)^{-1/2}\Cxyp( \Cyp+\lambda I)^{-1/2}}\\
&\hspace{1.5cm}+\norm{(\Cxp+\lambda I)^{-1/2}\Cxyp( \Cyp+\lambda I)^{-1/2}-(\ECxp+\lambda I)^{-1/2}\ECxyp( \ECyp+\lambda I)^{-1/2}}.
\end{align*}
Using Lemma 4 in \citep{ullah2023}, we get the following control on the "bias term" of the previous display:
\begin{align*}
     \norm{ (\Cxp)^{\dagger/2}\Cxyp (\Cyp)^{\dagger/2} -   (\Cxp+\lambda I)^{-1/2}\Cxyp( \Cyp+\lambda I)^{-1/2}}\leq 8 \sqrt{\lambda}.
\end{align*}
Next, using Lemma 6 in \citep{fukumizu2007} gives the following asymptotic rate of convergence on the "variance" part. Assume that $\lambda = \lambda_n \rightarrow 0$ as $n\rightarrow \infty$, then
\begin{align*}
 \norm{(\Cxp+\lambda I)^{-1/2}\Cxyp( \Cyp+\lambda I)^{-1/2}-(\ECxp+\lambda I)^{-1/2}\ECxyp( \ECyp+\lambda I)^{-1/2}} = O_{\PP}\left(\frac{1}{\sqrt{\lambda^3 \, n}} \right).
\end{align*}
Combining the last three displays and for the optimal choice $\lambda=\lambda_n= n^{-1/4}$, we obtain
\begin{align*}
\left| \norm{(\Cxp)^{\dagger/2}\Cxyp (\Cyp)^{\dagger/2}} -   \norm{( \ECxp+\lambda I)^{-1/2}\ECxyp( \ECyp+\lambda I)^{-1/2}} \right| = O_{\PP}(n^{-1/3}).
\end{align*}
According again to \citep{ullah2023}, in the most favorable scenario of finite-dimensional spaces with well-conditioned covariance matrices, we can obtain an improved control on the "variance" part of the order of magnitude $O_{\PP}(\lambda^{-1/2} n^{-1/2})$. Consequently, taking $\lambda=\lambda_n \asymp n^{-1/2} \rightarrow 0$ as $n\rightarrow \infty$, the estimation error improves from $O_{\PP}(n^{-1/3})$ to $O_{\PP}(n^{-1/2})$ in the best-case scenario. While these guarantees were obtained for the spectral norm, one can directly deduce such guarantees also for the HS-norm when the latent space is low-dimensional. This analysis highlights a fundamental limitation of using the score $\score(w)$. On the one hand, the Ridge regularization parameter $\lambda$ cannot be set too small to maintain the numerical stability of the method. On the other hand, the statistical analysis requires that $\lambda=\lambda_n$ converges sufficiently fast to $0$ as $n \rightarrow \infty$ to guarantees that the empirical score $\escore$ approximate the true objective $\score$. These two constraints are antagonistic and it is not clear that they can be simultaneously satisfied in practice. Hence, balancing between numerical stability and convergence to the true objective presents a challenging trade-off when using $\score$ for learning invariant representations. By contrast, our score $\rscore$ and its relaxed version do not suffer from this limitation. Namely, there is no need to sacrifice numerical stability for statistical accuracy or vice versa since $|\rscore(w) - \erscore(w)| = O_{\PP}(n^{-1/2})$ in all situations, as we prove it in Theorem \ref{thm:Shat-S} below.

Our goal is to derive  concentration guarantees for the empirical score $\erscore$ from the true score $\rscore$. We focus on time-homogeneous Markovian dynamical systems in the stationary regime with invariant measure $\im$ which was proposed in \citep{Lasota1994,Kostic2022}.

Recall the definition of the true and empirical scores
\begin{equation}
\rscore(\param):= \frac{\hnorm{\Cxyp}^2}{ \norm{\Cxp}\norm{\Cyp}}\quad \text{and}\quad 
\erscore(\param):= \frac{\hnorm{\ECxyp}^2}{ \norm{\ECxp}\norm{\ECyp}},
\end{equation} 
where, as before,
\[
\ECxp \!:=\! \tfrac{1}{n} 
{\sum_{i\in[n]}} \embXp{}(x_i) \embXp{}(x_i)^\top\!, \ECyp \!:=\! \tfrac{1}{n} 
{\sum_{i\in[n]}} \embYp{}(x'_i) \embYp{}(x'_i)^\top
\text{ and }\,
\ECxyp\! :=\! \tfrac{1}{n} 
{\sum_{i\in[n]}} \embXp{}(x_i) \embYp{}(x'_i)^\top\!\!.
\]
Denote by $\rho$ the joint distribution of $(X,X')$.

We assume that the embeddings are bounded almost surely, that is there exists an absolute constant $c$ such that
\begin{align}\label{eq:cond_bounded}
    \esssup_{x\sim\mu}\|\embXp{}(x)\|^2 \leq c, \qquad  \esssup_{x'\sim\mu'}\|\embYp{}(x')\|^2\leq c.
\end{align}
For any fixed $\param$ and any fixed $\delta\in (0,1)$, we assume that $n$ is large enough such that
\begin{align}\label{eq:cond_sample_size}
    \frac{4c}{3\,\norm{ \Cxp }\,n} \log\left(12r\delta^{-1}\right) + \sqrt{\frac{2}{\norm{\Cxp}\,n}\log(12r\delta^{-1})}< \frac{1}{3}.
\end{align}
Define
$$
\rate_n(\delta)= 
 \frac{4c}{3\,\norm{ \Cxp }\,n} \log\left(12r\delta^{-1}\right) + \sqrt{\frac{2}{\norm{\Cxp}\,n}\log(12r\delta^{-1})},
$$
and
$$
\rate''_n(\delta):= c^2\sqrt{\frac{5 \log(18 \delta^{-1})}{n}} + c^2\frac{C}{n},
$$
where $C$ is some absolute constant.

\begin{theorem}
\label{thm:Shat-S}
Let Conditions \eqref{eq:cond_bounded} and \eqref{eq:cond_sample_size} be satisfied. Then we get with probability at least $1-\delta$
\begin{align}
    \left\vert \rscore(\param) - \erscore(\param)   \right\vert \leq  \rscore(\param)\,\frac{3\rate_n(\delta)}{1-3\rate_n(\delta)}+ \frac{\rate''_n(\delta)}{\norm{\Cxp}\norm{\Cyp}(1-3\rate_n(\delta))}.
\end{align}    
\end{theorem}

\begin{proof}
By definition of $\rscore(\param)$ and $\erscore(\param)$, we have
\begin{align}
    &\left\vert  \erscore(\param) - \rscore(\param)\right\vert \leq \erscore(\param) \frac{\left\vert \norm{\ECxp}\norm{\ECyp} - \norm{\Cxp}\norm{\Cyp} \right\vert }{\norm{\Cxp}\norm{\Cyp}} + \frac{\left\vert \hnorm{\ECxyp}^2 - \hnorm{\Cxyp}^2 \right\vert }{\norm{\Cxp}\norm{\Cyp}}\notag\\
    &\hspace{1cm}\leq \vert \erscore(\param)  -\rscore(\param)\vert \frac{\left\vert \norm{\ECxp}\norm{\ECyp} - \norm{\Cxp}\norm{\Cyp} \right\vert }{\norm{\Cxp}\norm{\Cyp}} +  \rscore(\param) \frac{\left\vert \norm{\ECxp}\norm{\ECyp} - \norm{\Cxp}\norm{\Cyp} \right\vert }{\norm{\Cxp}\norm{\Cyp}}\notag\\
    &\hspace{5cm}+ \frac{\left\vert \hnorm{\ECxyp}^2 - \hnorm{\Cxyp}^2 \right\vert }{\norm{\Cxp}\norm{\Cyp}}.\notag 
\end{align}
Using \eqref{eq:Covbound2}, \eqref{eq:Covbound3} below, we get with probability at least $1-2\delta$,
$$
\frac{\left\vert \norm{\ECxp}\norm{\ECyp} - \norm{\Cxp}\norm{\Cyp} \right\vert }{\norm{\Cxp}\norm{\Cyp}} \leq \rate_n(\delta)+\rate'_n(\delta)+\rate_n(\delta)\rate'_n(\delta).
$$
Next, \eqref{eq:Cross-cov-interm1} gives with probability at least $1-\delta$
$$
\frac{\left\vert \hnorm{\ECxyp}^2- \hnorm{\Cxyp}^2  \right\vert }{\norm{\Cxp}\norm{\Cyp}} \leq \frac{\rate''_n(\delta)}{\norm{\Cxp}\norm{\Cyp}}.
$$
Under Conditions \ref{eq:cond_bounded} and \ref{eq:cond_sample_size}, we get
$$
\rate_n(\delta)=\rate'_n(\delta) = 
 \frac{4c}{3\,\norm{ \Cxp }\,n} \log\left(4r\delta^{-1}\right) + \sqrt{\frac{2}{\norm{\Cxp}\,n}\log(4r\delta^{-1})}<\frac{1}{3},
$$
and
$$
\rate''_n(\delta):= c^2\sqrt{\frac{5 \log(6 \delta^{-1})}{n}} + c^2\frac{C}{n}.
$$
An union bound gives with probability at least $1-3 \delta$,
\begin{align}
    \left\vert  \erscore(\param) - \rscore(\param)\right\vert 
    &\leq \rscore(\param)\,\frac{3\rate_n(\delta)}{1-3\rate_n(\delta)}+ \frac{\rate''_n(\delta)}{\norm{\Cxp}\norm{\Cyp}(1-3\rate_n(\delta))}.
\end{align}
Replacing $\delta$ with $\delta/3$, we get the result with probability $1-\delta$.
\end{proof}

To control the operator norm deviation of the empirical covariances $\ECxp$ and $\ECyp$ from their population counterparts, we use the following dimension-free version of \citep{minsker2017} of the non-commutative Bernstein inequality (see also Theorem 7.3.1 in \citep{tropp2012user} for an easier to read and slightly improved version) as well as an extension to self-adjoint Hilbert-Schmidt operators on separable Hilbert spaces.%

\begin{proposition}[\citep{minsker2017} and Theorem 7.3.1 in \citep{tropp2012user}]\label{prop:con_ineq_0}
Let $A_i$, $i\in[n]$ be i.i.d copies of a random Hilbert-Schmidt operator $A$ on separable Hilbert spaces. Let $\norm{A}\leq c$ almost surely, $\EE A =0$ and let $\EE[A^2]\preceq V$ for some  trace class operator $V$. Then with probability at least $1-\delta$ 
\begin{equation}\label{eq:con_ineq_0}
\left\|\frac{1}{n}\sum_{i\in[n]}A_i \right\|\leq \frac{2c}{3n} \mathcal{L}_A(\delta)+ \sqrt{\frac{2\norm{V}}{n}\mathcal{L}_A(\delta)},
\end{equation}
where
\[
\mathcal{L}_A(\delta):= \log\frac{4}{\delta}+ \log\frac{\tr(V)}{\norm{V}} .
\]
\end{proposition}

\begin{proposition}\label{prop:op_cov_bound}
Assume that $\bcon:=\sup_{x}\left\lbrace \norm{\embXp{}(x)}^2 \right\rbrace<\infty$.
Given $\delta>0$, with probability in the i.i.d. draw of $(x_i)_{i=1}^n$ from $\mX$, it holds that
\begin{equation}
\label{eq:Covbound}
    \PP\{ \norm{ \ECxp - \Cxp }/\norm{ \Cxp }\leq \rate_n(\delta) \} \geq 1-\delta,
\end{equation}

where
\begin{equation}\label{eq:eta}
\rate_n(\delta) := \frac{4\bcon}{3\,\norm{ \Cxp }\,n} \mathcal{L}(\delta) + \sqrt{\frac{2}{\norm{\Cxp}\,n}\mathcal{L}(\delta)}\quad\text{ and }\quad \mathcal{L}(\delta):= \log\frac{4r}{\delta}.
\end{equation}

\end{proposition}
\begin{proof}[Proof of Proposition \ref{prop:op_cov_bound}]
Proof follows directly from Proposition \ref{prop:con_ineq_0} applied to operators  $\embXp{}(x_i)\otimes\embXp{}(x_i)$ using the fact that $\Cxp = \EE\,\embXp{}(x_i)\otimes\embXp{}(x_i)$,
where we recall that $\embXp{}(x):=(\embXp{,1}(x),\ldots, \embXp{,r}(x)) \in\R^r$. Hence, we have the obvious upper bound $\frac{\tr(\Cxp)}{\norm{\Cxp}}\leq r.$
\end{proof}
We deduce from \eqref{eq:Covbound}, with probability at least $1-\delta$
\begin{equation}
\label{eq:Covbound2}
(1 - \rate_n(\delta))\norm{ \Cxp } \leq \norm{ \ECxp} \leq \norm{ \Cxp } (1 +\rate_n(\delta))\norm{ \Cxp },
\end{equation}
A similar result is also valid for $\ECyp$ provided that $\bconbis:=\sup_{\param,x}\left\lbrace \norm{\embYp{}(x)}^2 \right\rbrace<\infty$. Define
\begin{equation}\label{eq:etabis}
\rate'_n(\delta) := \frac{4\bconbis}{3\,\norm{ \Cxp }\,n} \mathcal{L}(\delta) + \sqrt{\frac{2}{\norm{\Cxp}\,n}\mathcal{L}(\delta)}\quad\text{ and }\quad \mathcal{L}(\delta):= \log\frac{4r}{\delta}.
\end{equation}
Then, with probability at least $1-\delta$
\begin{equation}
\label{eq:Covbound3}
(1 - \rate'_n(\delta))\norm{ \Cyp } \leq \norm{ \ECyp} \leq (1 +\rate'_n(\delta))\norm{ \Cyp },
\end{equation}

We study now the deviation of $\hnorm{\ECxyp}^2$ from $\hnorm{\Cxyp}^2$. We can essentially apply Theorem 3 in \citep{gretton2005measuring} with kernels $k_{\param}(x,x') = \langle \embXp{}(x),\embXp{}(x') \rangle$ and $l_{\param}(y,y') = \langle \embYp{}(y),\embYp{}(y') \rangle$. Note that these two kernels are essentially bounded. Indeed we have
\begin{align*}
    \esssup_{x,x'} \left\vert k_{\param}(x,x')\right\vert \leq \sup_{x}\|\embXp{}(x)\|^2\leq \bcon\quad\text{and}\quad \esssup_{y,y'} \left\vert l_{\param}(y,y')\right\vert \leq \sup_{x}\|\embYp{}(y)\|^2\leq \bconbis.
\end{align*}
Hence, for any $n\geq 2$ and $\delta>0$, we get with probability at least $1-\delta$,
\begin{align}
\label{eq:Cross-cov-interm1}
\left\vert  \hnorm{\ECxyp}^2- \hnorm{\Cxyp}^2   \right\vert  \leq  \epsilon''_n(\delta),
\end{align}
where 
$$
\epsilon''_n(\delta):= \bcon \bconbis\sqrt{\frac{5 \log(6 \delta^{-1})}{n}} + \bcon \bconbis\frac{C}{n},
$$
for some absolute constant $C>0$.

\subsection{Operator regression and prediction}
\label{app:prediction}

We next discuss how to design an estimator of the transfer operator $\TO\colon\LiiY\to\LiiX$ using the learned subspaces $\spHXp$ and $\spHYp$.  Namely, we estimate $\TO\approx \ETO_{\param} \colon \spHYp\to \spHXp$. The purpose of such estimation is to, given a initial state $x\in \spX$, predict the average evolution $\EE[ f(X')\,\vert\,X=x]$ of an observable $f\in\LiiY$. Remark that in the main text, we have just discussed the task when one takes $\spHYp=\spHXp$.

In what follows, let us assume that after the training we obtained $\param\in\spParam$ such that $\Cxp$ and $\Cyp$ are invertible, that is that $(\embXp{,j})_{j\in[r]}$  $(\embYp{,j})_{j\in[r]}$ form basis of the spaces $\spHXp$ and $\spHYp$, respectively. This means that the operators $\Enc{\param,}\colon \R^r\mapsto\spHXp$ and  $\Enc{\param}'\colon \R^r \mapsto \spHYp$ can be properly defined as partial isometries by $\Enc{\param}v = \embXp{}(\cdot)^\top v$ and $\Enc{\param}'v = \embYp{}(\cdot)^\top v$. So, every estimator can be written in the form $\ETO_{\param} = \Enc{\param}\EEstim (\Enc{\param}')^*$ for some $\EEstim\in\R^{r\times r}$. 
Different estimators can then computed from data  $\Data_n:=(x_i,x_i')_{i\in[n]}$ 
(either seen or unseen during training time). To elaborate on this, let us, as usual for kernel methods, define the sampling operators $\ESX\colon \spHXp\to\R^{n}$ and $\ESY\colon \spHYp\to\R^{n}$, given by $\ESX h = n^{-\frac{1}{2}} [ h(x_1)\,\ldots\, h(x_n)]^\top$, $f\in\spHXp$, and $\ESY g = n^{-\frac{1}{2}} [ g(x_1')\,\ldots\, g(x_n')]^\top$, $g\in\spHYp$. Notice that we can extend the domain of definition of these  operators via interpolation %
to arbitrary  functions $\spX\to\R$ that can be evaluated on a dataset, respectively. Hence, without possible confusion, when evaluating we can use $\ESX f$ and $\ESY f'$ even when $f\not\in\spHXp$ or $f'\not\in\spHYp$. 

Now, as shown in \citep{Kostic2022}, the empirical estimator $\ETO_{\param}$ of the transfer operator $\TO$ using dataset $\Data_n$ can be obtained via operator regression by minimizing the empirical risk 
\[
\hnorm{\ESY - \ESX \ETO_{\param} }^2 = \hnorm{\ESY - \ESX \Enc{\param}\EEstim (\Enc{\param}')^*}^2 = \hnorm{\ESY\Enc{\param}' - \ESX \Enc{\param}\EEstim }^2
\]
where the last equality holds since $\Enc{\param}'$ is a partial isometry. Therefore, the simple least square (LS) estimator is than obtained as $\EEstim:= ( \Enc{\param}^*\ESX^*\ESX \Enc{\param})^\dagger (\Enc{\param}^*\ESX^*\ESY\Enc{\param}') = (\ECxp)^\dagger\ECxyp$, or, equivalently, as $\ETO_{\param} = \Enc{\param}\EEstim(\Enc{\param}')^* = \ESX^* (\ESX\ESX^*)^\dagger\ESY$

Once the regression is performed, recalling that $\rvY$ is a $\Delta t = 1$ step ahead evolution of $\rvX$ we can use it to approximate 
$\EE[ f(\rvY)\,\vert\,\rvX = x] \approx (\ETO_{\param} f)(x) $ for $f\colon\spX\to\R$, as the following result shows. 

\begin{proposition}\label{prop:prediction}
Let $\ETO\colon\spHYp\to\spHXp$ be a LS estimator of $\TO$, then for every $x\in\spX$ and $f\in\spHXp$
\begin{equation}\label{eq:prediction}
(\ETO_{\param} f)(x) = \embXp{
}(x)^\top\, (\ECxp)^\dagger \embMXp [f(x_1')\,\vert\cdots\,\vert f(x_n')]^\top,
\end{equation}
where $\embMXp  := [\embXp{}(x_1)\,\vert\,\cdots\,\vert\,\embXp{}(x_n)]\in\R^{r\times n}$. 
\end{proposition}
\begin{proof}
The proof follows directly observing that $\embMXp^\top = \sqrt{n} [\ESX \embXp{,1}\,\vert\,\ldots\,\vert\,\ESX \embXp{,r}] = \sqrt{n} \ESX\Enc{\param}$. Namely, then 
\[
\embMXp [f(x_1')\,\vert\cdots\,\vert f(x_n')]^\top = (\ESX\Enc{\param})^* \ESY f = \Enc{\param}^* \ESX^*\ESY f = \Enc{\param}^* \ESX^*\ESY \Enc{\param}\Enc{\param}^*f = \ECxyp \Enc{\param}^*f 
\] 
and, hence 
\[
\embXp{
}(x)^\top\, (\ECxp)^\dagger \embMXp [f(x_1')\,\vert\cdots\,\vert f(x_n')]^\top = \Enc{\param} (\ECxp)^\dagger \ECxyp \Enc{\param}^*f = \ETO_\param f.
\]
\end{proof}
We remark that the previous result formally holds for $f\in\spHYp$, but it can be easily extended to functions in $\LiiY$ via interpolation.

\subsection{Dynamics mode decomposition and forecasting}
\label{app:forecasting}
Now we consider the problem of forecasting the process for several time steps in future using what is known as (extended) dynamic mode decomposition, which is based on the estimated eigenvalues and eigenfunctions of the transfer operator. As observed in the main body, this is meaningful only if the operator is an endomorphism on a function space, that is if it maps the space into itself. 

Hence, after training DPNet we will use just one representation $\embXp{}$ and its $r$-dimensional space of functions $\spHXp:=\Span(\embXp{,j})_{j\in[r]}$ to perform the operator regression, as explained in the previous section, and
obtain an estimator $\ETO_{\param} = \Enc{\param}\EEstim \Enc{\param}^*\colon \spHXp{}\to\spHXp{}$, for some matrix $\EEstim\in\R^{r\times r}$. Then, if 
$(\eeval_i, \elevec_i, \erevec_i)_{i\in[r]}\subset\C \times \C^r \times \C^r$ is a spectral decomposition of $\EEstim$, then $(\eeval_i, \Enc{\param}\elevec_i, \Enc{\param}\erevec_i)_{i\in[r]}$ is a spectral decomposition of $\ETO_{\param}$. In the following result we show how to compute dynamic mode decomposition of $\TO$ based on the estimator $\ETO_{\param}$ and use it for forecasting by approximating $\EE[ f(X_{t})\,\vert\,X_0 = x] \approx ((\ETO_{\param})^t f)(x)$, for $f\colon\spX\to\R$, $x\in\spX$.

\begin{proposition}\label{prop:forecasing}
Let $\ETO_{\param} = \Enc{\param}\EEstim\Enc{\param}\colon\spHXp\to\spHXp$ be rank $r$ LS estimator of $\TO_{\Delta t}$, for $\Delta t =1$. If $\EEstim=\sum_{i\in[r]}\eeval_i\erevec_i\,\elevec_i^*$ is the spectral decomposition of $\EEstim$, and $\erefun_i(x) := \embXp{}(x)^\top \erevec_i$ and $\elefun_i(x) :=  (\elevec_i)^*\embXp{}(x)$, $i\in[r]$, then for every $t\in\N$, every $x\in\spX$ and every $f\in\spHXp$ it holds that
\begin{equation}\label{eq:estimator_KMD}
((\ETO_{\param})^t f)(x) = \sum_{i\in[r]} \eeval_i^{t}\,\erefun_i (x)\,\elevec_i^*\Dec{\param}(f),
\end{equation}
where $\Dec{\param}(f):= \widehat{\Lambda}^{-1}(\ECxp)^\dagger \embMXp [f(x_1')\,\vert\dots\,\vert f(x_n')]^\top\in\R^{r}$.
\end{proposition}
\begin{proof}
First observe that 
\[
(\ETO_{\param})^t f= \Enc{\param}\EEstim^t\Enc{\param}^* f= \Enc{\param}\EEstim^{t-1}\Enc{\param}^*\Enc{\param}\EEstim\Enc{\param}^* f = \sum_{i\in[r]} \eeval_i^{t-1}\,(\erefun_i \otimes \elefun_i) \Enc{\param}\EEstim\Enc{\param}^* f.
\]
Hence, 
\[
((\ETO_{\param})^t f)(x)= \sum_{i\in[r]} \eeval_i^{t-1}\,\erefun_i(x) (\elevec_i^* \EEstim\Enc{\param}^*f) = \sum_{i\in[r]} \eeval_i^{t}\,\erefun_i(x) (\eeval^{-1}\elevec_i^* \EEstim\Enc{\param}^*f) 
\] 
and the rest of the proof follows as in Prop.~\ref{prop:prediction}. 
\end{proof}

We remark that $\elevec_i^*\Dec{\param}(f)$ is known as $i$-th Koopman mode of the observable $f$, and that in comparison to Encoder-Decoder approaches $\Dec{\param}(f)$ can be considered as a decoder when forecasting function $f\colon\spX\to\R$. Clearly this is easily extended to vector valued functions, and, hence, we can forecast the states by using $\elevec_i^*\Dec{\param}(\Id):= \widehat{\Lambda}^{-1}(\ECxp)^\dagger \embMXp [x_1'\,\vert\dots\,\vert x_n']^\top\in\R^{d}$, where $\spX\subset\R^d$.

\subsection{Extended Algorithm and Training Time}\label{app:train}

\begin{algorithm}
\caption{DPNets Training (Discrete)}\label{alg:dpnetsapp}
\begin{algorithmic}[1]
\Require data $\Data_n = (x_1,\dots, x_n)$, $\Data'_n = (x'_1,\dots,x'_n)$, metric distortion loss $\Reg$; optimizer $U$; DNNs $\psi_w, \psi'_w: \spX \to \R^r$; metric loss coefficient $\reg$; \# of steps $K$; minibatch size $m$.

\State Initialize DNN weights $w_1$

\For{$k=1$ to $K$},

    \State Sample minibatches $(y_1,\dots, y_m)$ from $\Data_n$, and $(y'_1,\dots,y'_m)$ from $\Data'_n$. 

    \State Compute empirical covariance matrices
\
    
$
\widehat{C}_{X}^{w_k} \gets \tfrac{1}{n}  \sum_{i=1}^m \psi_{w_k}{}(y_i) \psi_{w_k}{}(y_i)^\top, \widehat{C}_{X'}^{w_k} \gets \tfrac{1}{n}{\sum_{i=1}^m} \psi'_{w_k}{}(y'_i) \psi'_{w_k}{}(y'_i)^\top
$

$ 
\widehat{C}_{XX'}^{w_k} \gets \tfrac{1}{n} 
{\sum_{i=1}^m} \psi_{w_k}{}(y_i) \psi'_{w_k}{}(y'_i)^\top
$

    \If{DPNets-relaxed} 

        \State 
        $F(w_k)\gets \erscore^\reg_m(w_k) := \frac{\hnorm{\widehat{C}_{XX'}^{w_k}}^2}{ \norm{\widehat{C}_{X}^{w_k}}\norm{\widehat{C}_{X'}^{w_k}}}
- \reg \big(\Reg(\widehat{C}_{X}^{w_k}) + \Reg(\widehat{C}_{X'}^{w_k})\big)$ 

    \Else

        \State 
        $
        F(w_k) \gets
         \escore^\reg_m(w_k) := \hnorm{(\widehat{C}_{X}^{w_k})^{\frac{\dagger}{2}}\widehat{C}_{XX'}^{w_k} (\widehat{C}_{X'}^{w_k})^{\frac{\dagger}{2}}}^2 - \reg \big(\Reg(\widehat{C}_{X}^{w_k}) + \Reg(\widehat{C}_{X'}^{w_k})\big)
        $ 

    \EndIf %

    \State $w_{k+1} \gets U(w_{k}, \nabla F(w_k)$) where $\nabla F(w_k)$ is computed via backpropagation

\EndFor
\State \textbf{return} representations $\psi_{w_K}$, $\psi'_{w_K}$
\end{algorithmic}

\end{algorithm}

\Cref{alg:dpnetsapp} is an extended version of \Cref{alg:dpnets}.
The time complexity of computing the empirical scores $\escore^\reg_{m}(w)$, $\erscore^\reg_m(w)$ and (thanks to backpropagation) their gradients, is $O(m\,\text{Cost}(\psi_w) +mr^2 + r^3)$  and $O(m\,\text{Cost}(\psi_w) +mr^2)$ respectively, where $\text{Cost}(\psi_w)$ is the cost of one evaluation of $\psi_w$. 
Namely, computing the embeddings for $m$ samples costs $O(m\,\text{Cost}(\psi_w))$, computing the (cross)covariance matrices for $m$ samples costs $O(m r^2)$, computing the pseudoinverse via eignevalue decomposition costs $O(r^3)$, while computing the operator norm of the covariance matrices using e.g.\@ the Arnoldi iteration method costs $O(r^2)$.
We note that the cost of training VAMPNets \citep{mardt2018vampnets} is the of the same order as DPNets without relaxation, since evaluating the metric distortion loss is relatively cheap once we have the covariance matrices.

\section{SDE learning}\label{app:SDE}

We next prove the second main result on the optimization problem \eqref{eq:optimization_continious}.

\thmScoreCont*
\begin{proof}
 In view of Lem.~\ref{lm:partial_trace}, we only need to prove that $\tr(\proj{\spHXp}\LO\proj{\spHXp}) = \tr\left((\Cxp)^{\dagger}\Cxyd \right)$. To that end, we reason as in the proof of Lem.~\ref{lm:proj_to_cov} to obtain that $\proj{\spHXp} =  U Q (\Cxp)^{\dagger/2}Q^*\TS^*$, where $\TS\colon\spHXp \hookrightarrow \Wii$ is an injection and   $Q\colon\R^{r}\to\spHXp$ and $U\colon\spHXp\to\Wii$ are partial isometries. 
Moreover, recalling \eqref{eq:generator}, we have that 
\begin{align*}
    (Q^*\TS^* \LO \TS Q) & = \lim_{\Delta t\to 0^+} \frac{Q^*\TS^* (\TO_{\Delta t} - \Id) \TS Q}{\Delta t}\\
    & = \lim_{\Delta t\to 0^+} Q^*\left[ \int_{\spX} \im(dx) \fHXp(x) \otimes \frac{\int_{\spX} p_{\Delta t}(x,dx') (Q^*\fHXp(x') - Q^*\fHXp(x))}{\Delta t} \right] \\
    & =  \int_{\spX} \im(dx) \embXp{}(x) \otimes \left( \lim_{\Delta t\to 0^+} \frac{\int_{\spX} p_{\Delta t}(x,dx') (\embXp{}(x') - \embXp{}(x))}{\Delta t}\right)  \\
    & = \EE_{X\sim \im} \left[\embXp{}(X) \otimes d\embXp{}(X)\right] = \Cxyd, 
\end{align*}
where the last line follows from It\={o} formula  \citep[see e.g.][]{arnold1974}. 
Hence, using $\proj{\spHXp} =  U Q (\Cxp)^{\dagger/2}Q^*\TS^*$ we have that 
\[
\tr(\proj{\spHXp}\LO \proj{\spHXp} ) = \tr((\Cxp)^{\dagger/2} \Cxyd (\Cxp)^{\dagger/2}) = \tr((\Cxp)^{\dagger}\Cxyd),
\]
which completes the proof
\end{proof}

To conclude this section, we remark that in the continuous setting the  estimator $\ELO_{\param}\colon\spHXp\to\spHXp$ of $\LO$ on the learned space $\spHXp$ can be obtained via operator regression in a similar way as discussed in Sec.~\ref{sec:methods}. So, the LS estimator is given by matrix $\widehat{L}\,\,{=}\,\,(\ECxp)^\dagger \ECxyd$, and its spectral decomposition is 
\begin{equation}\label{eq:estimator_evd_cont}
\ELO_{\param}= \textstyle{\sum_{i\in[r]}} \eeval_i \,\erefun_i \otimes \elefun_i, \quad\text{ where }\quad \erefun_i(x) := \embXp{}(x)^\top \erevec_i\;\text{ and }\; \elefun_i(x) :=  (\elevec_i)^*\embXp{}(x),
\end{equation}
where $\widehat{L}=\sum_{i\in[r]}\eeval_i \erevec_i\,\elevec_i^\top$ is the spectral decomposition of the matrix $\widehat{L}$.

Hence, using that $\eval_i(\TO_{\Delta t})\,{=}\,\exp(\eval_i(\LO))$, we directly obtain the modal decomposition in the continuous time, 
\begin{equation}\label{eq:estimator_KMD_cont}
\EE[ f(X_{t})\,\vert\,X_0 = x] \approx \textstyle{\sum_{i\in[r]}} \exp(\eeval_i\,t)\,\erefun_i (x)\,\elevec_i^*\Dec{\param}(f),\quad t\in[0,+\infty)
\end{equation}
where $\Dec{\param}(f)\,{=}\, (\ECxp)^\dagger \embMXp [f(x_1\,)\,\vert\cdots\,\vert f(x_n)]^\top\,{\in}\,\R^{r}$ are the coefficients of the LS estimator of $f$ in $\spHXp$. As a final remark, note that in \eqref{eq:estimator_KMD_cont} we regress function $f$ onto $\spHXp$ using least squares with data $(x_i,f(x_i))$.

\section{Experiments}\label{app:exp}

\textbf{Hardware~} The experiments were performed on a workstation equipped with an Intel(R) Core\texttrademark i9-9900X CPU @ 3.50GHz, 48GB of RAM and a NVIDIA GeForce RTX 2080 Ti GPU. Due to RAM insufficiency, the Nystr\"om baseline reported in Table~\ref{tab:chignolin} was performed on a CPU node of a cluster with 2x AMD EPYC 7713 @ 2.0GHz and 512GB of RAM.

\textbf{Software~} All experiments and baselines have been implemented in Python 3.11 and Pytorch 2.0, the only exception being the PINNs baseline for the fluid flow experiment, for which we relied on the original implementation of the code. All the code to reproduce the experiments will be made openly available. 

\textbf{General remarks~} Every algorithm has been performed in \texttt{float32} single precision, fixing the random number generator seed where appropriate. We made sure that each algorithm was trained on the same combinations of input-output data, and for neural network models we used the same batches and number of epochs. The learning rate for each method was tuned independently by running a small number of steps at 100 equally spaced learning rates in the interval $(10^{-6}, 10^{-2})$ and selecting the best out of these. The best learning rates for each experiment are reported in \Cref{tab:lrates}

\begin{table}[]
\centering
\begin{tabular}{@{}l|lllll@{}}
\toprule
          & \textbf{DPNets} & \textbf{DPNets-relax} & \textbf{VAMPNets} & \textbf{DynAE} & \textbf{ConsAE} \\ \midrule
\textbf{Logistic}  & $1 \times 10^{-4}$   & $1 \times 10^{-4}$         & $3 \times 10^{-6}$     & -     & -      \\
\textbf{Fluid}     & Failed      & $1 \times 10^{-4}$        & Failed        & $9 \times 10^{-4}$  & $9 \times 10^{-4}$   \\
\textbf{MNIST}     & $9 \times 10^{-4}$   & $9 \times 10^{-4}$         & $9 \times 10^{-4}$     & $9 \times 10^{-4}$  & $9 \times 10^{-4}$   \\
\textbf{Chignolin} & Failed      & $1 \times 10^{-3}$         & Failed        & -     & -      \\
\textbf{Langevin}  & $1 \times 10^{-3}$  & -            & -        & -     & -      \\ \bottomrule
\end{tabular}
\caption{Best learning rates found by a grid search for each (experiment, method) pair. For all experiments, the grid was made by $100$ equally spaced values in the interval $(10^{-6}, 10^{-2})$.}\label{tab:lrates}
\end{table}

Once a representation was learned, we always used the Ordinary Least Squares estimator described in~\ref{sec:methods} to perform the subsequent evaluation tasks.

\subsection{Logistic map}\label{app:exp_logistic}
\textbf{Data~} We generated the data as explained in~\citet{Kostic2022} for a value $N = 20$ of the trigonometric noise. To train DPNets and VAMPNets we have sampled a trajectory of $2^{14} \approx 16000$ points. We used a batch size of $2^{13}$ points and trained for 500 epochs.

\textbf{Optimization~} Adam with learning rate tuned as explained in the general remarks.

\textbf{Architecture~} Multi layer perceptron of shape \texttt{Linear[64]$\,\to\,$Linear[128]$\,\to\,$Linear[64]$\,\to\,$Linear[feature\_dim]} with Leaky ReLU activations. We have set the dimension of the feature map to $r=7$ as the minimal dimension allowing to meaningfully learn the first three eigenvalues.

\textbf{Additional results: the role of feature dimension~} In the left panel of Fig.~\ref{fig:logistic_featuredim} we compare the eigenvalues and the singular values of $\TO$ for the noisy logistic map, while in the center and right panels we show the metrics reported in Tab.~\ref{tab:1} of the main text as a function of the feature dimension $r$. Notice how the singular values decay significantly slower than eigenvalues, a consequence of the fact that the transfer operator is not normal, i.e. $\TO\TO^*\neq \TO^*\TO$. Non-normality makes the estimation of the spectra of $\TO$ particularly sensitive, as captured by the pseudospectra, see \cite{TrefethenEmbree2020}. In Fig.~\ref{fig:logistic_eigs} we show how small estimation errors in operator norm (label on the contour lines) incur larger errors in the eigenvalue estimation (distance of the true eigenvalues to the contour lines in the complex plane). This means that for non-normal operators the estimation error needs to be typically much smaller than the modulus of the eigenvalues one wants to recover. For $r < 7$, the spectral error of every model is of the same order of the eigenvalues to approximate, as unequivocally shown in Fig.~\ref{fig:logistic_eigs}. At $r = 7$, DPNets-relaxed already give a decent approximation of the three leading eigenvalues, see also Fig.~\ref{fig:logistic_eigs}. From $r > 7$ onward, every model progressively yields reasonable estimations, with DPNets and Cheby-T quickly catching up with DPNets-relaxed. For the optimality gap, every model shows an improving trend by increasing the feature dimension $r$, with DPNets showing the strongest performance.

\begin{figure}[!ht]
  \centering \includegraphics[width=0.99\textwidth]{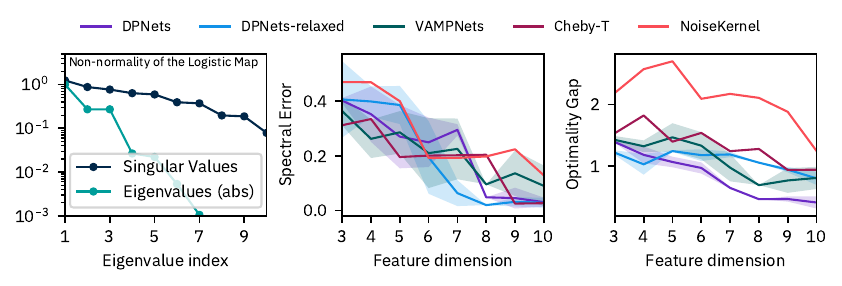}
    \caption{(Left) Decay of the singular values $\sigma_{i}(\TO)$ and of the the eigenvalues $|\lambda_{i}(\TO)|$ for the logistic map example. (Center, Right) Spectral error and Optimality gap as a function of the feature dimension $r$ for the baselines considered in Tab.~\ref{tab:1}.}
\label{fig:logistic_featuredim}
\end{figure}

\begin{figure}[!ht]
  \centering \includegraphics[width=0.99\textwidth]{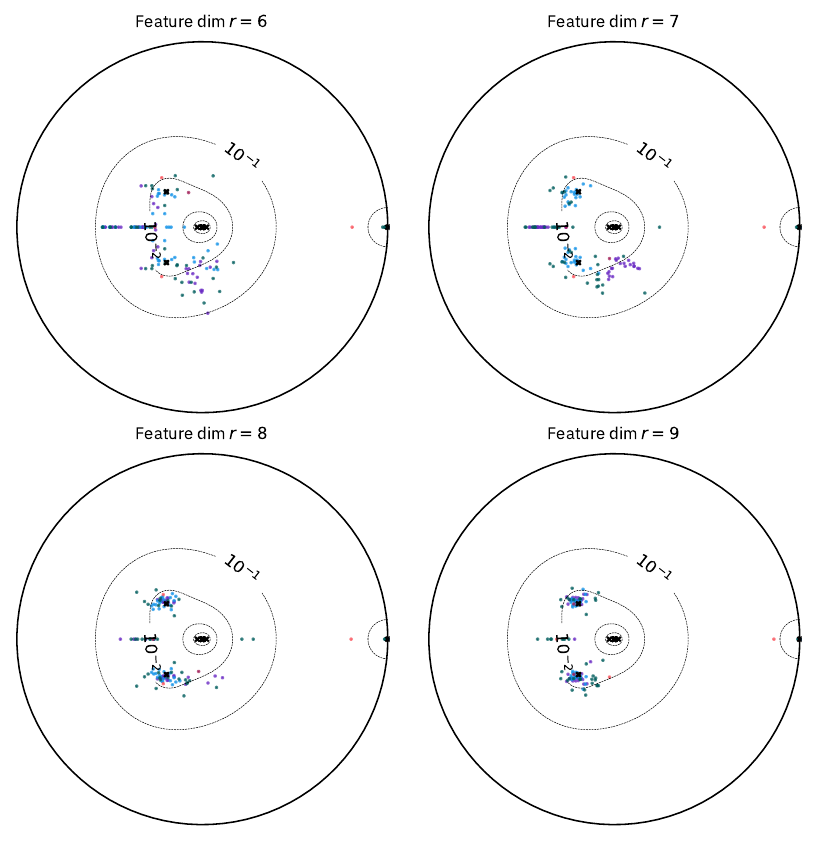}
    \caption{Estimation of the eigenvalues of the logistic map at different feature dimensions. For each model, we plot the eigenvalues estimated by 20 independent initial seeds. The color in each scatter plot follows the same color-coding of Fig.~\ref{fig:logistic_featuredim}. Black $\times$ represent reference eigenvalues, while the dashed contour lines show the pseudospectral regions when the estimation error ranges from $10^{-1}$ (outermost) to $10^{-4}$ (innermost). Note that for this problem pseudospectra indicates that the leading eigenvalue is easy to recover, while recovering eigenvalues close to zero is very hard.}
\label{fig:logistic_eigs}
\end{figure}

\begin{figure}[!ht]
  \centering
  \includegraphics[width=0.45\textwidth]{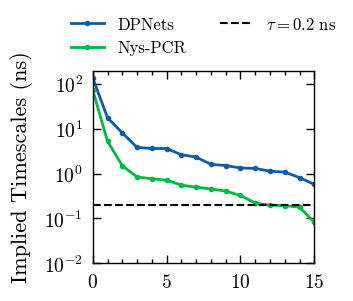}
    \caption{The implied timescale~\citep{mardt2018vampnets} associated to the eignvalues of the transfer operator. Everything which is below the simulation lagtime $\tau = 0.2$ ns, is just numerical noise.}
\label{fig:chignolin_timescales}
\end{figure}

\textbf{Hyperparameters~} Metric loss coefficient $\gamma=1$.

The reported results concern the average over 20 independent seed initializations.

\subsection{Fluid Flow past a cylinder}

\begin{figure}[!ht]
  \centering \includegraphics[width=0.99\textwidth]{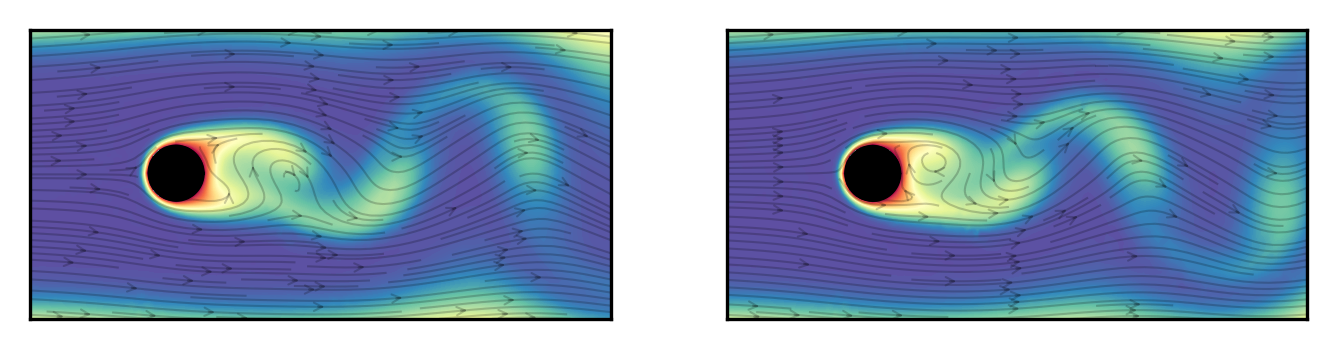}
    \caption{Two snapshots of the passive scalar concentration for the flow past a cylinder.}
\label{fig:fluid_dynamics}
\end{figure}
The data for this experiment are equally spaced sampled solutions of the Navier-Stokes equations~\citep{TrefethenEmbree2020} for an incompressible Newtonian fluid coupled with the transport equation $$\partial_{t}c + \boldsymbol{u}\cdot \nabla c = {\rm{Pec}}^{-1}\nabla^{2}c.$$ Here, ${\rm{Pec}}$ is the P\'eclet number, and $c: \R^3 \to \R$ is a field representing the concentration $c(t; x, y)$ of a scalar quantity which is transported by the fluid flow without influencing the fluid motion itself (e.g. a dye dissolved in water). These partial derivative equations are solved over a $100 \times 200$ regular grid. 

In this experiment we have only been able to train DPNets-relaxed. Indeed both DPNets (unrelaxed) and VAMPNets failed due to linear algebra errors arising in the back propagation step for the pseudo-inverse matrix.

\textbf{Data~} Available at \texttt{https://github.com/maziarraissi/HFM}. It consists of 201 snapshots: 160 used for training, the rest for testing. Has been standardized: each channel with its own mean and std. In Fig.~\ref{fig:fluid_dynamics} we show two snapshots from the training dataset.

\textbf{Optimization~}  Adam with learning rate tuned as explained in the general remarks. Full-batch training. 5000 total training iterations/epochs.
The training time of a full batch $\approx 39$ mins.

\textbf{Architecture~} MLP with layers of width \texttt{[128, 512, 1024, 512, 256, 64]} and ReLU activation function

\textbf{Hyperparameters~} Metric loss coefficient $\reg = 1$; 

\subsection{Continuous dynamics}
The implementation of this experiment is straightforward, and our results can be reproduced using the following informations. 

\textbf{Data~} Produced in-house with \text{JAX MD}. Both the dataset and the script to produce new trajectories will be released. The dataset consists of $10^5$ snapshots: $70\%$ used for training, $10\%$ for validation, $20\%$ for testing.

\textbf{Optimization~}
Adam with learning rate $10^{-3}$. Other parameters are the predefined values in \texttt{Optax}'s implementation. Batch size: 8192. 500 epochs. Training time is $\approx$ 2 mins.

\textbf{Architecture~} Multi layer perceptron with CeLu activation function. Dimension of the hidden layer: \texttt{[32, 64, 128, 128, 64, 4]}. 

\textbf{Hyperparameters~} Metric loss coefficient $\reg = 50$.

\subsection{Ordered MNIST}
\textbf{Data}~ out of the full MNIST dataset we generated a trajectory of 1000 steps as discussed in the main text, and evaluated the forecasting accuracy over 1000 different test initial conditions. In Tab. \ref{tab:time} we report the training time for DPNets, DPNets-relaxed and the baselines used. We observe that both our methods are the fastest during training. See also Fig. \ref{fig:digits} for the generated sequences of digits by the compared methods. \riccardo{could be nice to also have inference times}

\begin{table}[h]  %
    \centering    %
    \caption{Ordered MNIST Training times. Each model is run on CPU and we report mean $\pm$ std for 20 runs.}  %
    \label{tab:time}
    \begin{tabular}{r|c|c}
\hline
   \textbf{Model} &  \textbf{Fit Time (s)}  &  \textbf{Time per Epoch (s)}  \\
\hline
      DynamicalAE &    $0.571 \pm 0.034$    &       $0.057 \pm 0.003$       \\
  Oracle-Features &         $0.098$         &               -               \\
              DMD &         $0.333$         &               -               \\
  KernelDMD-Poly3 &         $4.914$         &               -               \\
 KernelDMD-AbsExp &         $0.832$         &               -               \\
    KernelDMD-RBF &         $0.776$         &               -               \\
           \textbf{DPNets} &    $0.046 \pm 0.004$    &       $0.049 \pm 0.001$       \\
\textbf{DPNets-relaxed} &    $0.043 \pm 0.004$    &       $0.051 \pm 0.004$       \\
\hline
\end{tabular}
\end{table}
\textbf{Optimization~} Adam with learning rate tuned as explained in the general remarks. Batches of 128 samples trained over 150 epochs.

\textbf{Architecture~} \texttt{Conv2d[16]$\,\to\,$ReLU$\,\to\,$MaxPool[2]$\,\to\,$ Conv2d[32]$\,\to\,$ReLU$\,\to\,$MaxPool[2]$\,\to\,$Linear[5]}.

For the Auto-Encoder baselines we used this architecture for the encoder and the ``reversed'' network constructed with transposed convolutions for the decoder.

\textbf{Hyperparameters~} Metric loss coefficient $\reg = 1$.

\begin{figure}[!ht]
  \centering \includegraphics[width=0.99\textwidth]{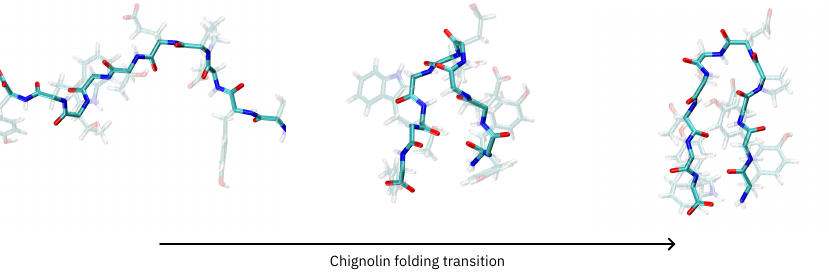}
    \caption{Three snapshots of Chignolin while undergoing a folding transition.}
\label{fig:chignolin_transition}
\end{figure}

\subsection{The metastable states of Chignolin}

In this experiment we build from the work in~\citep{Ghorbani2022} by employing our framework to learn the leading eigenfunctions associated to the dynamics of Chignolin, a folding protein, from a $106 \, \mu{\rm s}$ long molecular dynamics simulation sampled every $200 \, {\rm ps}$, totalling over half a million data points~\citep{LindorffLarsen2011}. We consider every heavy atom for a total of 93 nodes as well as a cutoff radius of 6 Angstroms giving an average of 30 neighbours for each atom. Compared to~\citep{Ghorbani2022}, which selects only the $\approx 20$ $C_{\alpha}$ atoms each with its first 5 neighbours, our experiment has therefore a much larger scale. Indeed,~\citep{Ghorbani2022} reports being able to train directly with the objective $\score^{0}$, while in our case we always encountered numerical errors, and we were able to only succesfully train the $\rscore^{\reg}$ objective. We parametrized the feature map with a graph neural network (GNN) model. GNNs currently are the state of the art in modeling atomistic systems \citep{Chanussot2021}, and allow one to elegantly incorporate the roto-translational and permutational symmetries prescribed by physics. Specifically, we train a SchNet~\citep{schutt2019schnetpack, schutt2023schnetpack} model with 3 interaction blocks, where in each block the latent atomic environment is 64-dimensional and the inter-atomic distances used for the message-passing step are expanded over 20 radial basis functions. After the last interaction block, each latent atomic environment is forwarded to a linear layer and then aggregated via averaging. %
The model has been trained for 100 epochs with an Adam optimizer and a learning rate of $10^{-3}$. %
We analyzed the eigenfunctions using the technique described in \citep{Novelli2022}, which links each metastable state to physically interpretable conformational descriptors. Our analysis aligns perfectly with \citep{Novelli2022}, where the slowest metastable state corresponds to the folding-unfolding transition and is linked to the distance between residues (1, 10) and (2, 9) located at opposite ends of the protein. Additionally, the immediately faster metastable state represents a conformational change within the folded state, characterized by the relative angle between residues 6 and 8. 

In Fig.~\ref{fig:chignolin_transition} we plot how the structure of Chignolin changes while performing a folding transition, while in Fig.~\ref{fig:chignolin_timescales} we plot the implied timescales of the dynamical modes of Chignolin as estimated by DPNets-relaxed and Nystr\"om-PCR.

This is by far the heaviest experiment of the paper, and we made use of the package \texttt{SchNetPack}~\citep{schutt2019schnetpack, schutt2023schnetpack} on multiple instances. In particular, we have used \texttt{SchNetPack} dataloaders and preprocessing transformations (casting to 32 bit precision and on-the fly computation of the distance matrix). Further, we have used \texttt{SchNetPack}'s implementation of the SchNet interaction block. To reproduce our results, the following informations may prove useful.

\textbf{Data~} The data was presented for the first time in~\citep{LindorffLarsen2011} and is freely available for non-commercial use upon request to DeShaw research.  Dataset of 524743 snapshots. Each graph is composed by the 93 heavy atoms. The edges are formed only if two atoms are less than 5 Angstroms distant. The average number of edges is 28.

\textbf{Optimization~} Adam with learning rate $10^{-3}$. Other parameters are the predefined values in \texttt{Torch}'s implementation.
Batch size: 192. 100 epochs. Training time: $\approx 11$ hrs.

\textbf{Architecture~} SchNet with 3 blocks, 20 RBF functions expansions, 64 latent dimension. At the output of SchNet, the hidden variables associated to the nodes are averaged and forwarded to a dense layer with 16 final output features. 

\textbf{Hyperparameters~} Metric loss coefficient $\reg = 0.01$. For the Nystr\"om baseline we used $M = 5000$ inducing points.

\begin{sidewaysfigure}[ht]
    \includegraphics[width=\textwidth]{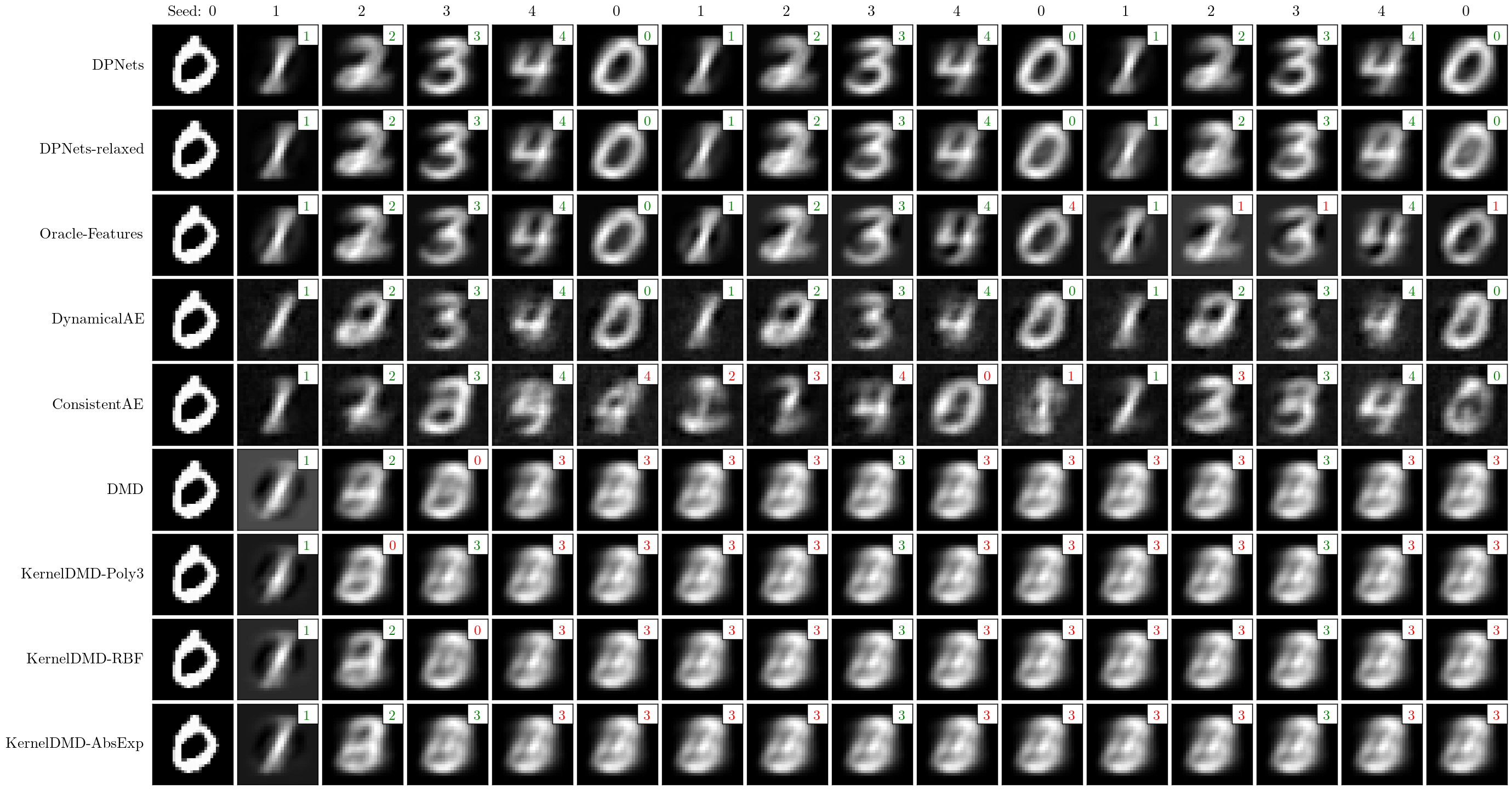}
    \caption{Sample of the states predicted by the different models starting from a seed image of the digit 0. The green label in the upper right corner of each image indicates proper classification, while the red one means miss-classification.}
    \label{fig:digits}
\end{sidewaysfigure}
\end{document}